\documentclass{article}

\usepackage{microtype}
\usepackage{graphicx}
\usepackage{subcaption}
\usepackage{booktabs} 

\usepackage{hyperref}

\usepackage[accepted]{icml2020_preprint}

\icmltitlerunning{The Geometry of Sign Gradient Descent}

\usepackage{todonotes}
\usepackage{amsmath, amssymb, amsthm}
\usepackage{mathtools}
\usepackage{wrapfig}
\usepackage{enumitem}

\definecolor{mydarkblue}{rgb}{0,0.08,0.45}

\DeclareMathOperator{\sign}{sign}
\DeclareMathOperator*{\argmin}{arg\, min}
\DeclareMathOperator*{\argmax}{arg\, max}
\DeclareMathOperator{\diag}{diag}
\DeclareMathOperator{\norm}{norm}
\DeclareMathOperator{\trace}{tr}
\newcommand{\R}{\mathbb{R}}
\newtheorem{theorem}{Theorem}
\newtheorem{lemma}{Lemma}
\newtheorem{proposition}{Proposition}
\newtheorem{definition}{Definition}
\def\xx{\boldsymbol x}
\def\yy{\boldsymbol x^\prime}
\def\zz{\boldsymbol z}
\def\ee{\boldsymbol e}
\def\ss{\boldsymbol s}
\def\g{\boldsymbol g}
\def\mm{\boldsymbol m}
\def\vv{\boldsymbol v}
\def\HH{\boldsymbol H}
\def\II{\boldsymbol I}
\def\LL{\boldsymbol L}

\def\defas{\coloneqq}

\def\EE{{\mathbb E}}

\newcommand{\ggdop}[2]{#2^{#1}}

\begin{document}

\twocolumn[
\icmltitle{The Geometry of Sign Gradient Descent}



\icmlsetsymbol{equal}{*}

\begin{icmlauthorlist}
\icmlauthor{Lukas Balles}{tue,mpi}
\icmlauthor{Fabian Pedregosa}{goo}
\icmlauthor{Nicolas Le Roux}{goo,mcg}
\end{icmlauthorlist}

\icmlaffiliation{mpi}{MPI for Intelligent Systems, Tuebingen}
\icmlaffiliation{tue}{University of Tuebingen}
\icmlaffiliation{goo}{Google Research, Brain Team}
\icmlaffiliation{mcg}{Mila, McGill University, Montreal}

\icmlcorrespondingauthor{Lukas Balles}{lballes@tue.mpg.de}

\icmlkeywords{Machine Learning, ICML}

\vskip 0.3in
]



\printAffiliationsAndNotice{}  

\begin{abstract}

Sign-based optimization methods have become popular in machine learning due to their favorable communication cost in distributed optimization and their surprisingly good performance in neural network training.
Furthermore, they are closely connected to so-called adaptive gradient methods like Adam.
Recent works on signSGD have used a non-standard ``separable smoothness'' assumption, whereas some older works study sign gradient descent as steepest descent with respect to the $\ell_\infty$-norm.
In this work, we unify these existing results by showing a close connection between separable smoothness and $\ell_\infty$-smoothness and argue that the latter is the weaker and more natural assumption.
We then proceed to study the smoothness constant with respect to the $\ell_\infty$-norm and thereby isolate geometric properties of the objective function which affect the performance of sign-based methods.
In short, we find sign-based methods to be preferable over gradient descent if (i) the Hessian is to some degree concentrated on its diagonal, and (ii) its maximal eigenvalue is much larger than the average eigenvalue.
Both properties are common in deep networks.
\end{abstract}

\section{Introduction}

We consider an unconstrained, continuous optimization problem, $\min_{\xx\in\mathbb{R}^d} f(\xx)$, with a differentiable and lower-bounded objective $f\colon \R^d \to \R$.
The prototypical optimization algorithm to solve such problems is gradient descent (GD), which iteratively updates $\xx_{t+1} = \xx_t - \alpha_t \nabla f_t$ with $\nabla f_t = \nabla f(\xx_t)$.
In machine learning, computing $\nabla f_t$ is often inefficient and instead one resorts to stochastic gradient descent (SGD), $\xx_{t+1} = \xx_t - \alpha_t \g_t$, where $\g_t$ is a stochastic gradient estimate which can be obtained at lower cost, e.g., by data subsampling.

Several recent works have considered the sign gradient descent (signGD) method and its stochastic counterpart (signSGD)
\begin{align}
    \label{eq:signgd_without_norm}
    \xx_{t+1} &= \xx_t - \alpha_t \sign(\nabla f_t), \\
    \label{eq:signsgd_without_norm}
    \xx_{t+1} &= \xx_t - \alpha_t \sign(\g_t),
\end{align}
where the $\sign$ is applied elementwise.
In particular, these methods have been studied in the context of distributed optimization where they conveniently reduce the communication cost to a single bit per gradient coordinate~\citep[e.g.,][]{seide2014onebit, bernstein2018signsgd, karimireddy2019error}.
SignSGD is also of interest due to a connection to the popular Adam method \citep{kingma2014adam} on which we will expand in Section~\ref{sec:signgd_and_adam}.

\paragraph{Analysis of Sign-Based Methods}

Multiple authors \citep{kelner2014almost,carlson2015stochastic,karimi2016linear} have analyzed variants of sign-based methods under the assumption of smoothness with respect to the $\ell_\infty$-norm (maximum norm), i.e.,
\begin{equation}
    \Vert \nabla f(\yy) - \nabla f(\xx) \Vert_1 \leq L_\infty \Vert \yy - \xx \Vert_\infty
\end{equation}
for all $\xx, \yy \in \R^d$ with smoothness constant $L_\infty>0$.
On the other hand, \citet{bernstein2018signsgd} have analyzed signSGD under a non-standard smoothness assumption that there are constants $l_1\dotsc, l_d > 0$ such that
\begin{equation}
    \label{eq:separable_smoothness_def}
    f(\yy) \leq f(\xx) + \langle \nabla f(\xx), \yy-\xx\rangle + \frac{1}{2} \sum_i l_i (\yy_i - \xx_i)^2
\end{equation}
for all $\xx, \yy\in\R^d$.
Follow-up works refining the analysis also adopted this assumption \citep{bernstein2019signsgd,safaryan2019stochastic}, which we will refer to as \emph{separable} smoothness to emphasize that the quadratic term \emph{separates} over coordinates with individual constants $l_i$.

Reiterating all existing convergence results would be beyond the scope of this paper, but it is crucial to understand that the convergence rates of sign-based methods are governed by $L_\infty$ in the papers based on $\ell_\infty$-smoothness and by $\sum_i l_i$ in those based on the separable smoothness assumption.

\paragraph{Contributions}

The separable smoothness assumption seems to add an unnecessary level of granularity since neither the algorithm itself nor its analysis uses the individual values $l_i$.
This paper clarifies the relationship between separable smoothness and $\ell_\infty$-smoothness.
We show that the convergence results based on separable smoothness also hold under $\ell_\infty$-smoothness if $L_\infty = \sum_i l_i$ and that the latter is a strictly weaker assumption.
This unifies all existing results on sign-based methods under the umbrella of $\ell_\infty$-smoothness.

We then proceed to analyze the geometric meaning of $\ell_\infty$-smoothness.
We tie the corresponding smoothness constant $L_\infty$ to properties of the Hessian and show that it is favorable if the Hessian fulfills two conditions: (i) some degree of ``diagonal concentration'' and (ii) the maximal eigenvalue being much larger than the average eigenvalue.
Notably, these properties have repeatedly been observed in deep learning training tasks.
Our analysis thus provides a possible explanation of the empirical success of sign-based methods---and by extension Adam---in deep learning.
The dependence on the diagonal concentration of the Hessian, which relates to the axis-alignment of the objective, is in stark contrast to the Euclidean smoothness constant $L_2$, which controls the convergence speed of (stochastic) gradient descent.

\section{Sign Gradient Descent and Adam}
\label{sec:signgd_and_adam}

\begin{figure}
    \centering
    \includegraphics[width=\columnwidth]{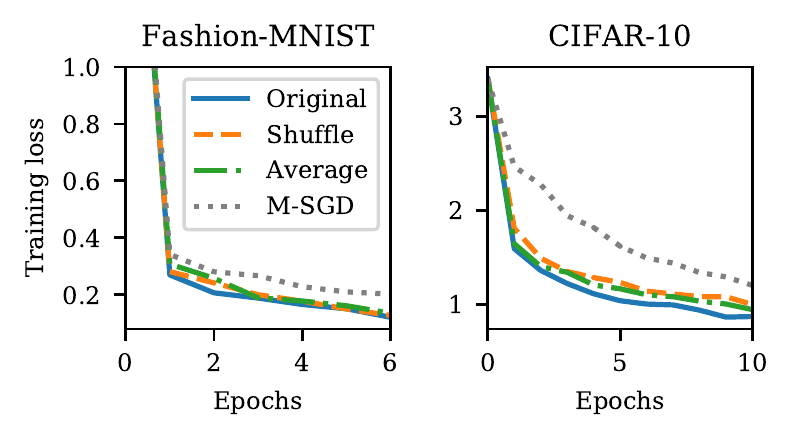}
    \caption{Original Adam compared to two variants based on the decomposition in Eq.~\eqref{eq:Adam_decomposition}.
    The factors $\boldsymbol{\gamma}_t$ are (i) randomly \emph{shuffled} or (ii) \emph{averaged} across coordinates.
    We add SGD with momentum for comparison.
    Both variants perform very similarly to original Adam, suggesting that Adam is primarily a sign-based method and the elementwise adaptivity plays a secondary role.
    Experimental details may be found in Appendix~\ref{apx:experimental_details}.}
    \label{fig:shuffle_Adam_mnist}
\end{figure}
This work primarily aims to provide a deeper understanding of sign-based optimization methods, but these insights may also be pertinent to the very popular Adam optimizer \citep{kingma2014adam}, which has been tied to signSGD by multiple authors \citep{balles2018dissecting, bernstein2018signsgd}.
As additional motivation for or study of sign-based methods, we briefly review this connection and provide additional corroborating evidence.

Adam maintains moving averages of stochastic gradients, $\mm_t  = \beta_1 \mm_{t-1} + (1-\beta_1) \g_t$, and their elementwise squares, $\vv_t = \beta_2 \vv_{t-1} + (1-\beta_2) \g_t^2$, and updates in the direction $-\mm_t / (\sqrt{\vv_t} + \varepsilon)$, where $\varepsilon>0$ is a small constant guaranteeing numerical stability.
\citet{bernstein2018signsgd} point out that Adam reverts to signSGD for $\beta_1, \beta_2,\varepsilon \rightarrow 0$.
\citet{balles2018dissecting} investigate the relationship more generally by rewriting Adam's update direction for $\varepsilon=0$ as
\begin{equation}
    \label{eq:Adam_decomposition}
    \frac{\mm_t}{\sqrt{\vv_t}} = \underbrace{ \left( 1 + \frac{\vv_t - \mm_t^2}{\mm_t^2} \right)^{-1/2}}_{=: \boldsymbol{\gamma}_t} \sign(\mm_t).
\end{equation}
Based on this decomposition,~\citet{balles2018dissecting} interpret Adam as a momentum version of signSGD with elementwise variance-based damping terms.\footnote{The term $\vv_t - \mm_t^2$ can be interpreted as an estimate of the elementwise variances of $\g_t$. Hence, each coordinate of $\boldsymbol{\gamma}_t$ is a damping factor in  $[0, 1]$ which is monotonically decreasing as a function of the gradient variance in that coordinate.}
Both \citet{balles2018dissecting} and \citet{bernstein2018signsgd} demonstrate that signSGD with momentum and Adam often have very similar practical performance on deep learning tasks.

We push the analysis further and experimentally investigate the relative importance of the variance-based damping ($\boldsymbol{\gamma_t}$ in Eq.~\ref{eq:Adam_decomposition}) and the sign aspect.
To that end, we compare standard Adam to two variants based on this decomposition:
\begin{itemize}[topsep=.2\topsep, itemsep=.2\itemsep, leftmargin=*]
    \item \emph{Shuffled}, where we use the update direction $\boldsymbol{\tilde{\gamma}}_t \sign(\mm_t)$ where $\boldsymbol{\tilde{\gamma}}_t$ is a vector that contains the elements of $\boldsymbol{\gamma}_t$ in a randomly shuffled order.
    \item \emph{Averaged}, where the update direction is $\bar{\gamma}_t \sign(\mm_t)$ where $\bar{\gamma}_t$ is the average value of $\boldsymbol{\gamma}_t$, resulting in momentum-signSGD with an adaptive scalar step size.\footnote{For ease of implementation, we actually shuffle/average separately for each ``variable'' (weight matrix, convolution filter, bias vector) since optimizers in deep learning frameworks are commonly implemented to act independently on each such variable.
}
\end{itemize}
Figure~\ref{fig:shuffle_Adam_mnist} depicts results for the training of simple CNN architectures on the Fashion-MNIST \citep{xiao2017fashion} and CIFAR-10 \citep{krizhevsky2009learning} datasets.
We see that both variants perform very similarly to the original Adam algorithm, corroborating the view that Adam is primarily a scaled variant of signSGD with momentum and that the elementwise adaptivity plays a secondary role.

\section{Separable Smoothness and \texorpdfstring{$\ell_\infty$}{linf}-Smoothness}
\label{sec:unifying}

After positioning sign-based methods in the context of other popular methods, we now clarify the relationship between separable smoothness and $\ell_\infty$-smoothness in order to unify existing convergence results.
We briefly review the concept of smoothness with respect to arbitrary norms and the associated steepest descent methods in Section~\ref{sec:smoothness_and_steepest_descent}.
Section~\ref{sec:replacing_separable_smoothness} details how $\ell_\infty$-smoothness can replace separable smoothness in the analysis of sign-based optimization methods.
We conclude with remarks on the consequences of this unification.

\subsection{Smoothness and Steepest Descent}
\label{sec:smoothness_and_steepest_descent}

Smoothness is a standard assumption in optimization and means that the gradient function is Lipschitz, i.e., $\Vert \nabla f(\yy) - \nabla f(\xx) \Vert_2 \leq L_2 \Vert \yy-\xx\Vert_2$ for some positive scalar $L_2$.
The crucial significance of this assumption is that it gives rise to local quadratic bounds on $f$:
\begin{equation}
    \label{eq:euclidean_smoothness_quadratic_bound}
    f(\yy) \leq f(\yy) + \langle \nabla f(\xx), \yy - \xx\rangle + \frac{L_2}{2} \Vert \yy - \xx\Vert_2^2.
\end{equation}
This bound motivates gradient descent; fixing $\xx$ and minimizing w.r.t.~$\yy$ yields the update $\yy = \xx - L_2^{-1} \nabla f(\xx)$.

\subsubsection{Smoothness w.r.t.~Arbitrary Norms}

The aforementioned notion of smoothness is based on the Euclidean norm and can be generalized to arbitrary norms.
We say $f$ is $L$-smooth w.r.t.~a norm $\Vert \cdot \Vert$ if
\begin{equation}
    \label{eq:general_smoothness}
    \Vert \nabla f(\yy) - \nabla f(\xx) \Vert_\ast \leq L \Vert \yy -\xx  \Vert
\end{equation}
for all $\xx, \yy\in\mathbb{R}^d$.
Here, $\Vert \cdot \Vert_\ast$ denotes the dual norm of $\Vert \cdot \Vert$, defined as $\Vert \zz \Vert_\ast \defas \max_{\Vert \xx \Vert\leq 1} \langle \zz, \xx\rangle$.
Table~\ref{tab:ggd_methods} lists the dual norm pairs for the methods under consideration in this paper.
The standard case of Euclidean smoothness falls under this definition since the Euclidean norm is dual to itself.

Due to the equivalence of norms on $\mathbb{R}^d$, a function that is smooth with respect to one norm is also smooth with respect to any other norm.
However, the tightest possible smoothness constant,
\begin{equation}
    \label{eq:tight_smoothness_constant}
    L \defas \sup_{\xx\neq \yy} \frac{\Vert \nabla f(\yy) - \nabla f(\xx)\Vert_\ast}{\Vert \yy-\xx\Vert},
\end{equation}
will depend on the choice of norm.
In the following, we will always assume $L$ to be given by Eq.~\eqref{eq:tight_smoothness_constant}.
This constant governs the convergence speed of the corresponding steepest descent method, which we will define next.

\subsubsection{Steepest Descent}

As in the Euclidean case, smoothness gives rise to a local quadratic bound on the function around a point $\xx$: 
\begin{lemma}
\label{lemma:general_smoothness_improvement_bound}
If $f$ is $L$-smooth w.r.t.~$\Vert\cdot\Vert$, then
\begin{equation}
    f(\yy) \leq f(\xx) + \langle \nabla f(\xx), \yy-\xx\rangle + \frac{L}{2} \Vert \yy-\xx \Vert^2
\end{equation}
for all $\xx, \yy\in\R^d$.
\end{lemma}
\begin{proof}
All proofs may be found in Appendix~\ref{apx:proofs}.
\end{proof}
The steepest descent method with respect to the norm~$\Vert\cdot\Vert$ iteratively minimizes this upper bound:
\begin{equation}
    \label{eq:steepest_descent_def}
    \xx_{t+1} \in \argmin_{\xx\in\R^d} \left( \langle \nabla f_t, \xx- \xx_t\rangle  + \frac{L}{2} \Vert \xx- \xx_t\Vert^2 \right).
\end{equation}
This minimizer need not be unique, in which case steepest descent is to be understood as choosing \emph{any} solution.

\begin{table}[t]
    \centering
    \footnotesize
    \begin{tabular}{llll}
    \textbf{Method} & \textbf{Norm} & \textbf{Dual} &  \textbf{Update direction} \\
    \hline
    Gradient descent & $\Vert \cdot \Vert_2$ & $\Vert \cdot \Vert_2$ & $\nabla f$ \\
    SignGD & $\Vert \cdot \Vert_\infty$ & $\Vert \cdot \Vert_1$ & $\Vert \nabla f\Vert_1 \sign(\nabla f)$ \\
    Coordinate descent & $\Vert \cdot \Vert_1$ & $\Vert \cdot \Vert_\infty$ & $\nabla f_{i_\text{max}} \ee^{(i_\text{max})}$ \\
    Block-normlzd. GD & $\Vert \cdot\Vert^\mathcal{B}_\infty$ & $\Vert \cdot\Vert^\mathcal{B}_1$ & see Appendix~\ref{apx:steepest_descent}
    \end{tabular}
    \caption{A few steepest descent methods. The table lists the used norm $\Vert\cdot\Vert$, its dual $\Vert\cdot\Vert_\ast$, and the resulting update direction.
    Coordinate descent and block-normalized gradient descent are discussed in Appendix~\ref{apx:steepest_descent}.}
    \label{tab:ggd_methods}
\end{table}

We have already seen that gradient descent is steepest descent with respect to the Euclidean norm.
As noted by \citet{kelner2014almost} and \citet{carlson2015stochastic}, steepest descent with respect to the maximum norm gives rise to a version of sign gradient descent, namely
\begin{equation}
    \label{eq:signgd_with_norm}
    \xx_{t+1} = \xx_t - \frac{1}{L_\infty} \Vert \nabla f_t \Vert_1 \sign (\nabla f_t).
\end{equation}
This is equivalent to Eq.~\eqref{eq:signgd_without_norm} up to the scaling with the gradient norm which may be  subsumed in the step size.

\emph{Side note.} The steepest descent framework encompasses many other well-known methods, see Table~\ref{tab:ggd_methods}.
For example, steepest descent w.r.t.~the $\ell_1$-norm yields a version of coordinate descent.
An interesting observation is that a block-wise extension of sign gradient descent arises as steepest descent w.r.t.~a block-wise maximum norm.
Variants of block-wise normalization have recently found application in deep learning~\citep{yu2017block, ginsburg2019stochastic} with blocks corresponding to layers.
We discuss this further in Appendix~\ref{apx:steepest_descent}.

\subsubsection{Convergence of Steepest Descent}
\label{sec:ggd_convergence}
The convergence of steepest descent methods bases upon the following Lemma, which guarantees an improvement in function value in each step.
\begin{lemma}
\label{lemma:steepest_descent_improvement}
Let $f$ be $L$-smooth w.r.t.~$\Vert \cdot \Vert$.
Then steepest descent (Eq.~\ref{eq:steepest_descent_def}) satisfies
\begin{equation}
    f(\xx_{t+1}) \leq f(\xx_t) - \frac{1}{2L} \Vert \nabla f(\xx_t)\Vert_\ast^2.
\end{equation}
\end{lemma}

This implies various convergence results, which we discuss in Appendix~\ref{apx:steepest_descent}.
Generally, all steepest descent methods will enjoy the same rate of convergence, but Lemma~\ref{lemma:steepest_descent_improvement} shows the significance of (i) the smoothness constant, which we want to be small, and (ii) the dual gradient norm, which we want to be large.
These two aspects will play a role when we compare sign gradient descent and gradient descent in Section~\ref{sec:gd_vs_signgd}.

\subsection{Separable Smoothness and \texorpdfstring{$\ell_\infty$}{linf}-Smoothness}
\label{sec:replacing_separable_smoothness}

We now show that $\ell_\infty$-smoothness can replace separable smoothness for the analysis of sign-based methods by showing that
\begin{itemize}[topsep=.2\topsep, itemsep=.2\itemsep]
    \item[(i)] separable smoothness with constants $l_1,\dotsc, l_d$ implies $\ell_\infty$-smoothness with constant $L_\infty=\sum_i l_i$, and
    \item[(ii)] convergence results based on separable smoothness also hold under the latter, weaker assumption.
\end{itemize}

While separable smoothness is directly defined by Eq.~\eqref{eq:separable_smoothness_def} in existing works, it is easily embedded in the framework of Section~\ref{sec:smoothness_and_steepest_descent} as $1$-smoothness w.r.t.~the norm $\Vert\cdot\Vert_{\LL}$ where $\LL=\diag(l_1, \dotsc, l_d)$ and $\Vert \zz \Vert_{\LL}^2 \defas \left( \sum_i l_i z_i^2 \right)$.
The bound given by Lemma~\ref{lemma:general_smoothness_improvement_bound} then coincides with Eq.~\eqref{eq:separable_smoothness_def}.
With that, we can establish statement (i).
\begin{proposition}
\label{proposition:separable_smoothness_implies_linf}
If $f$ is $1$-smooth w.r.t.~$\Vert\cdot\Vert_{\LL}$, then $f$ is $(\sum_i l_i)$-smooth w.r.t.~the maximum norm.
\end{proposition}

Regarding statement (ii), we note that the separable smoothness assumption in form of  Eq.~\eqref{eq:separable_smoothness_def} enters existing convergence proofs exclusively with $\xx = \xx_t, \yy = \xx_{t+1}$.
The simple but essential observation is that, since sign-based updates have the same magnitude in each coordinate, $(\sum_i l_i)$-smoothness w.r.t.~the maximum norm yields the exact same bound.
\begin{proposition}
\label{proposition:linf_smoothness_replaces_separable_smoothness}
Let $\xx\in \mathbb{R}^d$, $\ss\in\{-1, 1\}^d$ and $\alpha>0$.
Both separable smoothness with constants $l_1,\dotsc, l_d$ and $\ell_\infty$-smoothness with constant $L_\infty = \sum_i l_i$ imply
\begin{equation}
    f(\xx + \alpha \ss) \leq f(\xx) + \alpha \langle \nabla f(\xx), \ss \rangle + \frac{\alpha^2}{2} \sum_i l_i.
\end{equation}
\end{proposition}
Since all existing convergence results start from exactly this bound they hold under either of the two assumptions.

In summary, separable smoothness adds a level of granularity that sign-based algorithms can not exploit and that is not necessary for their analysis.
Max-norm smoothness removes this unnecessary granularity, is a strictly weaker assumption, and is naturally tied to sign-based methods via the steepest descent formalism.
We thus argue that $\ell_\infty$-smoothness should be used for the analysis of sign-based methods.

\subsection{Consequences of the Unification}

The results established in the previous subsection relax the assumptions of some previous works and allow for a better comparability of existing results.
We can adapt results based on separable smoothness to use $\ell_\infty$-smoothness simply by replacing $\sum_i l_i$ with $L_\infty$.
As an example, Theorem~1 of \citet{bernstein2018signsgd} shows convergence to a stationary point for signSGD on non-convex problems in the ``large-batch setting'', where a mini-batch size of $T$ is used to perform $T$ iterations.
Adapted to $\ell_\infty$-smoothness, the rate reads
\begin{equation}
    \EE\left[\frac{1}{T} \sum_{t=0}^{T-1} \Vert \nabla f_t \Vert_1\right]^2 \leq \frac{1}{T}\left[\sqrt{L_{\infty}} \left(f_0 - f^\star + \frac{1}{2}\right) + C_1 \right]^2.
\end{equation}
with a constant $C_1$ depending on the level of gradient noise.
\citet{bernstein2018signsgd} contrast this with a rate for SGD achieved under similar assumptions, namely 
\begin{equation}
    \EE\left[\frac{1}{T} \sum_{t=0}^{T-1} \Vert \nabla f_t \Vert_2^2\right] \leq \frac{1}{T}\left[2 L_2 \left(f_0 - f^\star\right) + C_2 \right].
\end{equation}
For details refer to \citet{bernstein2018signsgd}.

This can now easily be compared with other results, for example, those arising from the steepest descent framework.
In the deterministic setting, norm-scaled sign gradient descent (Eq.~\ref{eq:signgd_with_norm}) achieves a non-convex rate of
\begin{equation}
    \frac{1}{T} \sum_{t=0}^{T-1} \Vert \nabla f_t \Vert_1^2 \leq \frac{ 2 L_\infty (f_0 - f^\star)}{T},
\end{equation}
whereas gradient descent achieves
\begin{equation}
\frac{1}{T} \sum_{t=0}^{T-1} \Vert \nabla f_t \Vert_2^2 \leq \frac{ 2 L_2 (f_0 - f^\star)}{T} \; ,
\end{equation}
see Proposition~\ref{proposition:steepest_descent_convergence_smooth} in Appendix~\ref{apx:steepest_descent}.

Our unification makes clear that the max-norm smoothness constant $L_\infty$ crucially affects the convergence speed of sign-based methods, irrespective of the setting (stochastic vs.~deterministic) and the precise version (Eq.~\ref{eq:signgd_without_norm} vs.~Eq.~\ref{eq:signgd_with_norm}).
It is thus critical to establish a better understanding of the geometric meaning of $\ell_\infty$-smoothness, which we will tackle in the next section.

\section{Understanding \texorpdfstring{$\ell_\infty$}{linf}-Smoothness}
\label{sec:understanding_linf}

The previous sections have established the significance of smoothness w.r.t.~the maximum norm for sign-based optimization methods, showing that the corresponding smoothness constant $L_\infty$ crucially affects their convergence speed.
We now turn our attention to the ``meaning'' of this constant.
While we have a good intuition for the Euclidean smoothness constant---an upper bound on the eigenvalues of the Hessian---this is lacking for smoothness w.r.t.~the maximum norm.
Our goal in this section is to understand which properties of the objective function affect the constant $L_\infty$.
We first introduce a Hessian-based formulation of general smoothness constants, generalizing the aforementioned result for Euclidean smoothness.
We then show that $L_\infty$ depends on both the eigenvalues of the Hessian as well as its degree of ``diagonal concentration'', which corresponds to the axis-alignment of the objective.
Next, we contrast this more explicitly with Euclidean smoothness, pinpointing conditions under which the objective function has a favorable $L_\infty$ constant \emph{relative} to $L_2$, suggesting a susceptibility to sign-based optimization methods.
Finally, we discuss how these insights relate back to the separable smoothness condition.

\subsection{Smoothness as a Bound on the Hessian}
\label{sec:smoothness_hessian_bound}

The definition of smoothness in Eq.~\eqref{eq:general_smoothness} is unwieldy.
For the Euclidean norm, a more intuitive characterization is widely known:
A twice-differentiable function $f$ is $L_2$-smooth w.r.t.~the Euclidean norm if and only if $\Vert \nabla^2 f(\xx)\Vert_2 \leq L_2$ for all $\xx\in\R^d$.
Here, $\Vert \cdot\Vert_2$ for matrices denotes the spectral norm, given by the largest-magnitude eigenvalue for symmetric matrices.
To facilitate our discussion of the $\ell_\infty$-smoothness constant, the following proposition generalizes this Hessian-based formulation of smoothness constants.
It shows that smoothness with respect to any other norm likewise arises from a bound on the Hessian, but in different matrix norms.
\begin{proposition}
\label{proposition:smoothness_from_hessian_bound}
For any norm $\Vert \cdot \Vert$ on $\mathbb{R}^d$, we define the matrix norm
\begin{equation}
    \label{eq:induced_matrix_norm}
    \Vert \HH \Vert \defas \max_{\Vert \xx\Vert\leq 1} \Vert \HH \xx \Vert_\ast \, .
\end{equation}
A twice-differentiable function $f\colon \R^d\to\R$ is $L$-smooth w.r.t.~$\Vert\cdot\Vert$ if and only if $\Vert \nabla^2 f(\xx)\Vert \leq L$ for all $\xx$.
\end{proposition}

We are not aware of prior works showing this for the general case.
For the Euclidean norm, Proposition~\ref{proposition:smoothness_from_hessian_bound} gives back the familiar spectral norm bound.

\subsection{\texorpdfstring{$L_\infty$}{Linf} is Sensitive to Axis-Alignment}

By Proposition~\ref{proposition:smoothness_from_hessian_bound}, $L_\infty$ is determined by
\begin{equation}
    \label{eq:matrix_norm_signgd}
    L_\infty = \sup_{\xx\in\R^d} \Vert \nabla^2 f(\xx)\Vert_{\infty, 1}, \quad
    \Vert \HH \Vert_{\infty, 1} \defas \max_{\Vert \xx\Vert_\infty\leq 1} \Vert \HH \xx \Vert_1.
\end{equation}
Unfortunately, computing $\Vert \HH\Vert_{\infty, 1}$ is NP-hard~\citep{rohn2000computing}, but we can gain some intuition on a two-dimensional quadratic where all relevant quantities are easily found in closed form (see Appendix~\ref{apx:example_2d_quadratic} for details).
As depicted in Fig.~\ref{fig:ellipses_linf}, $L_\infty$ does not only depend on the eigenvalues, but also on the axis-alignment of the objective, which is encoded in the diagonal concentration of the Hessian.

We can generalize this insight through the following upper bound.
\begin{proposition}
\label{proposition:linf_hessian_norm_upper_bound}
Let $\HH\in\R^{d\times d} $ be nonzero, positive semi-definite with eigenvalues $\lambda_1, \dotsc, \lambda_\text{d}$. Then
\begin{equation}
    \Vert \HH\Vert_{\infty, 1} \leq \rho_\text{diag}(\HH)^{-1} \sum_{i=1}^d \lambda_i
\end{equation}
with $\quad \rho_\text{diag}(\HH) \defas \sum_i \vert H_{ii} \vert / \sum_{i, j} \vert H_{ij} \vert$.
\end{proposition}

The quantity $\rho_\text{diag}(\HH) \in [d^{-1},1]$ measures the degree of diagonal concentration.
Proposition~\ref{proposition:linf_hessian_norm_upper_bound} thus indicates that $L_\infty$ will depend both on the axis-alignment of the Hessian as well as its eigenvalues.
This is in stark contrast to Euclidean smoothness, for which the relevant matrix norm is invariant to rotations.
We will make this comparison more explicit in the next subsection.

Proposition~\ref{proposition:linf_hessian_norm_upper_bound} only applies to positive semi-definite matrices. We provide a similar, albeit slightly looser, bound for all symmetric matrices:
\begin{proposition}
\label{proposition:linf_hessian_norm_upper_bound_extended}
Let $\HH\in\R^{d\times d}$ be symmetric with eigenvalues $\lambda_1, \dotsc, \lambda_d$ and orthonormal eigenvectors $\vv^{(1)},\dotsc, \vv^{(d)}$. Then
\begin{equation}
    \Vert \HH \Vert_{\infty, 1} \leq \sum_{i=1}^d \vert \lambda_i\vert \; \Vert \vv^{(i)}\Vert_1^2 \; .
\end{equation}
\end{proposition}
In this bound, the dependency on the axis-alignment appears through $\Vert \vv^{(i)}\Vert_1^2$.
Since the eigenvectors have unit $\ell_2$-norm, their squared $\ell_1$-norm is bounded by $1 \leq \Vert \vv^{(i)}\Vert_1^2 \leq d$.
The lower bound is attained if $\vv^{(i)}$ is perfectly axis-aligned.
The upper bound is attained if all elements of $\vv^{(i)}$ are of equal magnitude.
Due to the weighting by $\vert \lambda_i \vert$, we see that the axis-alignment particularly matters for eigenvectors associated with eigenvalues of large magnitude.
Both bounds are tight for axis-aligned matrices but can be substantially loose for very non-axis-aligned ones.

These \emph{upper} bounds on $\Vert\HH\Vert_{\infty, 1}$ identify \emph{sufficient} conditions which favor sign gradient descent.
We note that the sensitivity to axis-alignment also manifests itself in a lower bound. Indeed, for any pair $(\lambda_i, \vv^{(i)})$ we have
\begin{equation}
\Vert \HH\Vert_{\infty, 1} \geq \frac{\Vert \HH\vv^{(i)}\Vert_1}{\Vert \vv^{(i)}\Vert_\infty} = \vert \lambda_i\vert \, \frac{\Vert \vv^{(i)}\Vert_1}{\Vert \vv^{(i)}\Vert_\infty}.
\end{equation}
Again, the ratio $\Vert\vv^{(i)}\Vert_1 / \Vert\vv^{(i)}\Vert_\infty\in [1, d]$ can be seen as a measure of axis-alignment of $\vv^{(i)}$.

\begin{figure*}
    \begin{subfigure}[c]{0.49\textwidth}
    \includegraphics{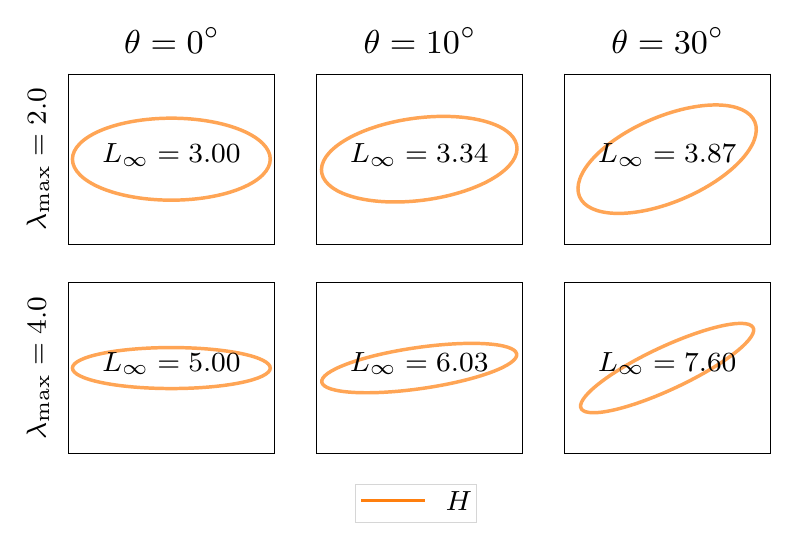}
    \caption{}
    \label{fig:ellipses_linf}
    \end{subfigure}
    \hfill
    \begin{subfigure}[c]{0.49\textwidth}
    \includegraphics{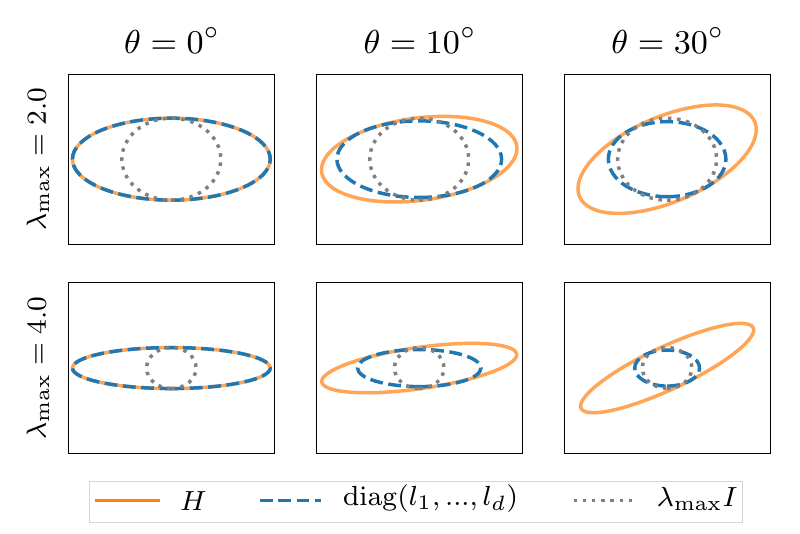}
    \caption{}
    \label{fig:ellipses_sep}
    \end{subfigure}
    \caption{For $\HH\in\R^{2\times 2}$, we plot a contour line of $f(\xx)=\frac{1}{2}\xx^T\HH\xx$, which forms an ellipse with principal axes given by the eigenvectors of $\HH$ and axis lengths given by the inverse eigenvalues.
    We fix $\lambda_\text{min}= 1$ and vary $\lambda_\text{max} > 1$ as well as the angle $\theta$ between the principal axes of the ellipse and the coordinate axes.
    (Mathematical details can be found in Appendix~\ref{apx:example_2d_quadratic}.)\\[1pt]
    (a) The $\ell_\infty$-smoothness constant $L_\infty=\Vert \HH\Vert_{\infty, 1}$ is sensitive to the axis-alignment of the objective.
    This is in contrast to the Euclidean smoothness constant, which is simply given by $L_2=\lambda_\text{max}$.\\[1pt]
    (b)
    Separable smoothness bounds $\HH$ by a diagonal matrix, $\diag(l_1,\dotsc, l_d) \succeq \HH$, corresponding to an axis-aligned ellipse that lies fully within the $\HH$-ellipse.
    The ``best'' bounding ellipse is given by Eq.~\eqref{eq:Lsep_def}.
    This bound changes with the axis-alignment, becoming both smaller and more circular (i.e., larger and more similar $l_i$) as we rotate further from the coordinate axes.
    In contrast to that, Euclidean smoothness bounds $\HH$ by $\lambda_\text{max} \II\succeq \HH$, i.e., a rotation-invariant circle.
    }
\end{figure*}

\subsection{Comparing \texorpdfstring{$L_\infty$}{Linf} and \texorpdfstring{$L_2$}{L2}}

We have seen that the $\ell_\infty$-smoothness constant governs the convergence of sign gradient descent (and related methods) whereas the Euclidean smoothness constant governs the convergence of gradient descent.
We now want to compare these two constants.
Assume $f$ to be $L_2$-smooth w.r.t.~$\Vert \cdot \Vert_2$ and $L_\infty$-smooth w.r.t.~$\Vert \cdot \Vert_\infty$ in the ``tight'' sense of Eq.~\eqref{eq:tight_smoothness_constant}.
Basic inequalities between the norms yield (see Appendix~\ref{apx:proofs})
\begin{equation}
    \label{eq:general_bound_linf_l2}
    L_2 \leq L_\infty \leq dL_2
\end{equation}
irrespective of $f$.
The smoothness constant governing the convergence speed of sign gradient descent will always be larger (worse) than the one pertinent to gradient descent.
This does \emph{not} mean that signGD is always worse than GD.
The smoothness constant is only one factor influencing the convergence speed; Lemma~\ref{lemma:steepest_descent_improvement} shows that the dual gradient norm plays a role as well.
As we will discuss in Section~\ref{sec:gd_vs_signgd}, this norm is larger for signGD which can make up for the disadvantage of a larger smoothness constant.

Here, we want to characterize under which circumstances $L_\infty$ \emph{tends} to be small---in particular much smaller than its worst case of $dL_2$---such that sign-based methods can be competitive with gradient descent.
Propositons~\ref{proposition:linf_hessian_norm_upper_bound} and \ref{proposition:linf_hessian_norm_upper_bound_extended} imply the following two conditions:
\begin{itemize}[topsep=.2\topsep, itemsep=.2\itemsep, leftmargin=*]
	\item Firstly, we need the Hessian to exhibit some degree of diagonal concentration or, equivalently, to have somewhat axis-aligned eigenvectors.
    \item Secondly, the sum of absolute eigenvalues should be much smaller than its worst case of $d \lambda_\text{max}$ or, equivalently, the average absolute eigenvalue should be much smaller than the maximal one:
    \begin{equation}
        \bar{\lambda} \defas \frac{1}{d} \sum_i \vert \lambda_i\vert \ll \max_i \vert \lambda_i \vert =: \lambda_\text{max}.
    \end{equation}
\end{itemize}

Notably, both properties have independently been found to be present in neural network training.
Multiple papers studying the Hessian in neural network training objectives \citep{chaudhari2016entropy, papyan2018full, ghorbani2019investigation, li2019hessian} find that the spectrum is dominated by a small number of \emph{outlier eigenvalues} that far exceed the bulk of eigenvalues, which are concentrated close to zero.
This is exactly the setting that will lead to $\bar{\lambda} \ll \lambda_\text{max}$.
The question of the diagonal concentration of the Hessian has been studied as early as \citeyear{becker1988improving} by \citeauthor{becker1988improving}.
More recently, \citet{adolphs2019ellipsoidal} have confirmed some degree of diagonal concentration in contemporary neural architectures.
In light of our analysis above, these observations suggest that the geometry of optimization problems encountered in deep learning lends itself to sign-based optimization methods (and its relatives such as Adam).

\subsection{Relationship to Separable Smoothness}

A natural question is how separable smoothness fits into this Hessian-based formulation and---given its close connection to $\ell_\infty$-smoothness---how it reflects the sensitivity to axis alignment.
\citet{bernstein2018signsgd} note that, for twice-differentiable $f$, separable smoothness results from a diagonal bound on the Hessian, i.e., $- \LL \preceq \nabla^2 f(\xx) \preceq \LL$ for all $x\in\R^d$ with $\LL =\diag(l_1,\dotsc, l_d)$.
It is tempting to think of the values $l_1,\dotsc, l_d$ merely as bounds on the eigenvalues of $\nabla^2 f(\xx)$, but we will see in the following that these values also depend on the axis-alignment of the Hessian.

With ``$\preceq$'' being only a partial ordering there is no clear definition of the \emph{tightest} diagonal bound.
Since the performance of sign-based methods depends on $\sum_i l_i$, we will choose the bound which minimizes this sum.
With that, we can establish the following result:
\begin{proposition}
\label{proposition:linf_lsep}
Let $\HH\in\R^{d\times d}$ be positive semi-definite with eigenvalues $\lambda_1\dotsc, \lambda_d$ and define
\begin{align}
    L_\infty(\HH) &\defas \max_{\Vert \xx\Vert_\infty\leq 1} \Vert \HH\xx \Vert_1, \label{eq:Linf_def} \\
    L_\text{sep}(\HH)  &\defas \min_{l_i \geq 0} \sum_i l_i \quad \text{s.t.} \quad \HH \preceq \diag(l_i). \label{eq:Lsep_def}
\end{align}
Then
\begin{equation}
    L_\infty(\HH) \leq L_\text{sep}(\HH) \leq \rho_\text{diag}(\HH)^{-1} \sum_i \lambda_i.
\end{equation}
\end{proposition}
Hence, $L_\text{sep}$ upper bounds $L_\infty$, reflecting again that $\ell_\infty$-smoothness is the weaker of the two conditions.
At the same time, $L_\text{sep}$ is upper-bounded by the same quantity that appears in Proposition~\ref{proposition:linf_hessian_norm_upper_bound}, reflecting the sensitivity to axis-alignment.
Figure~\ref{fig:ellipses_sep} illustrates this on a two-dimensional quadratic example.

\emph{Side note.} The two-dimensional example also reveals another shortcoming of the analysis based on separable smoothness.
To contrast the performance of sign-based methods with that of (stochastic) gradient descent, it is assumed that the convergence speed of the latter is controlled by $l_\text{max}$, which implicitly assumes $l_\text{max} = \lambda_\text{max}$.
This may be misleading since $l_\text{max}$ can exceed $\lambda_\text{max}$, see Appendix~\ref{apx:example_2d_quadratic}.
Of course, we could deviate from the definition in Eq.~\eqref{eq:Lsep_def} and choose $l_i \equiv \lambda_\text{max}$ to guarantee $l_\text{max}=\lambda_\text{max}$, but the resulting $L_\text{sep}$ would be very unfavorable for sign(S)GD.

\section{Gradient Descent vs Sign Gradient Descent}
\label{sec:gd_vs_signgd}

We found in Section~\ref{sec:understanding_linf} that $L_\infty$ always exceeds $L_2$ but identified conditions under which $L_\infty\ll d L_2$.
In this section, we want to explicitly compare gradient descent and sign gradient descent.
To facilitate this comparison, we make two restrictions.
First, we consider the non-stochastic setting in order to isolate the dependency of the two methods on the \emph{geometry} of the objective rather than the stochasticity.
Second, we consider the norm-scaled version of sign gradient descent (Eq.~\ref{eq:signgd_with_norm}).
This variant is not usually applied to, say, deep learning tasks, for multiple possible reasons which we discuss in Appendix~\ref{apx:normalized_steepest_descent}.
We argue that, in the smooth, non-stochastic setting, the norm-scaled version is natural and preferable since (i) it arises as steepest descent w.r.t.~the maximum norm, and (ii) unlike the version in Eq.~\eqref{eq:signgd_without_norm}, it converges with a constant step size.

Hence, we are comparing two steepest descent methods.
By Lemma~\ref{lemma:steepest_descent_improvement}, their guaranteed improvement in function value at $\xx\in\R^d$ is given by
\begin{align}
    \mathcal{I}_\text{GD}(\xx) & \defas \frac{\Vert \nabla f(\xx) \Vert_2^2}{L_2}, \\ \mathcal{I}_\text{signGD}(\xx) & \defas \frac{\Vert \nabla f(\xx) \Vert_1^2}{L_\infty}.
\end{align}
We now compare the two methods based on these improvement guarantees, thus taking into account the dual gradient norm in addition to the smoothness constant.
Basic norm inequalities, $d \Vert\nabla f(\xx)\Vert_2^2 \geq \Vert \nabla f(\xx) \Vert_1^2 \geq \Vert\nabla f(\xx) \Vert_2^2$, show that this potentially favors sign gradient descent.

Following~\citet{bernstein2018signsgd} we define
\begin{equation}
    \phi(\zz) \defas \frac{\Vert \zz \Vert_1^2}{d \Vert \zz \Vert_2^2} \in [d^{-1}, 1],
\end{equation}
which measures the density of the vector $\zz$, i.e., how evenly its mass is distributed across coordinates.
With that, we can write
\begin{equation}
\label{eq:signgd_gd_improvement_ratio}
    \begin{split}
        \mathcal{R}(\xx) & \defas \frac{\mathcal{I}_\text{signGD}(\xx)}{\mathcal{I}_\text{GD}(\xx)} = \phi(\nabla f(\xx)) \frac{d L_2}{L_\infty}.
    \end{split}
\end{equation}
If in addition to $L_\infty \ll dL_2$, for which we have identified conditions in Section~\ref{sec:understanding_linf}, we encounter dense gradients, then $\mathcal{R}(\xx)$ is large and sign gradient descent makes faster progress than gradient descent.

We emphasize that this is a \emph{local} comparison at $\xx\in\R^d$ due to the dependence on the gradient $\nabla f(\xx)$.
It is tempting to try and lower-bound the gradient density---that is, to assume $\phi(\nabla f(\xx)) \geq \phi_g \gg d^{-1}$ for all $\xx$---in order to make global statements.
\citet{bernstein2018signsgd} assume such a lower bound when contrasting signGD and GD in their analysis.
However, $\phi(\nabla f(\xx))$ can easily be shown to attain $d^{-1}$ even on quadratics.
Any non-trivial lower bound would thus have to be restricted to the trajectory of the optimizer and take into account the initialization, which seems out of reach at the moment.
The effect of the gradient norm thus remains an empirical question which we now address.

\paragraph{Quadratic Experiments}

\begin{figure*}
    \centering
    \includegraphics{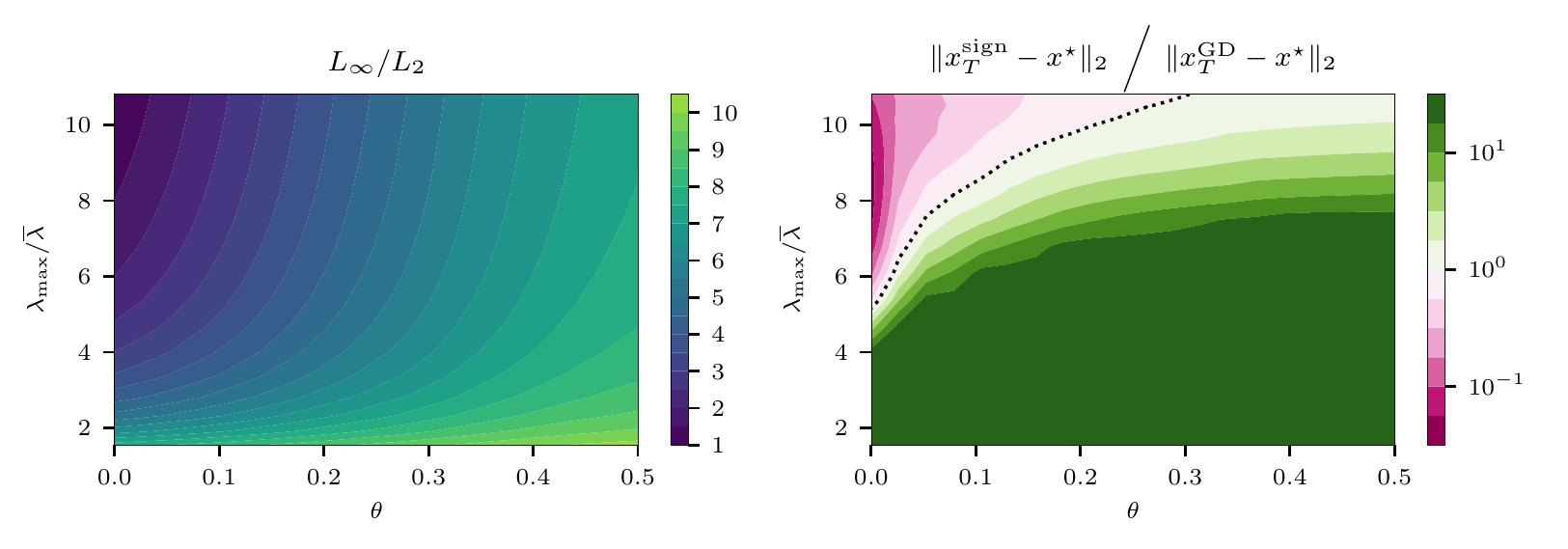}
    \caption{We consider quadratic objectives varying across two axes: $\lambda_\text{max} / \bar{\lambda} $ as well as a rotation value $\theta$.
    The left plot depicts the ratio of the two relevant smoothness constants.
    $L_\infty$ is sensitive to $\theta$ and grows relative to $L_2=\lambda_\text{max}$ as the problem becomes less axis-aligned.
    The right plot depicts the relative performance of gradient descent and sign gradient descent on these problems.
    GD drastically (the colormap is clipped) outperforms signGD for mildly-conditioned (small $\lambda_\text{max}/\bar{\lambda}$) and non-axis-aligned (large $\theta$) problems.
    However, sign gradient descent is preferable for problems with high $\lambda_\text{max}/\bar{\lambda}$, given that they have some degree of axis-alignment (small $\theta$). The dashed line represents equal performance of both algorithms.}
    \label{fig:quadratics}
\end{figure*}

We give a simple toy experiment to illustrate and verify the conditions under which signGD outperforms gradient descent.
We  consider synthetic quadratic problems of moderate dimension such that we can compute and control all relevant quantities.
We generate Hessians with varying $\bar{\lambda} / \lambda_\text{max}$ and axis alignment.
We set the eigenvalues as $\boldsymbol\Lambda = \diag(1, 1, \dotsc, 1, \lambda_\text{max})$.
To control the axis-alignment, we rotate the eigenvectors by some ratio $\theta$ in the direction prescribed by a randomly-drawn rotation matrix.
We can simply think of $\theta$ as a degree of rotation; the technical details can be found in Appendix~\ref{apx:experimental_details}.
For each Hessian, we compute the two smoothness constants $L_2 = \lambda_\text{max}$ and  $L_\infty = \Vert \HH \Vert_{\infty, 1}$.
We then run $T=100$ iterations of gradient descent (with $\alpha=1/L_2$) and sign gradient descent (with $\alpha=1/L_\infty$) and compute the distance to the optimum $\Vert \xx_T - \xx^\ast\Vert_2^2$ as a scale-invariant performance measure.
We average over repeated runs with $\xx_0\sim \mathcal{N}(0, \II)$ to marginalize out the effect of initialization.
The results are depicted in Fig.~\ref{fig:quadratics} and confirm the findings of Sections \ref{sec:understanding_linf} and \ref{sec:gd_vs_signgd}.
The $L_\infty$ constant---and consequently the performance of signGD---is sensitive to the axis-alignment of $\HH$ and suffers as we increase $\theta$.
For problems with $\lambda_\text{max} \gg \bar{\lambda}$ that are somewhat axis-aligned, sign gradient descent outperforms gradient descent, even on these simple quadratic problems.

\section{Conclusion}
\label{sec:conclusion}

In this paper, we made two main contributions.
First, we unified the assumptions of existing works on both batch and stochastic sign-based methods by clarifying the relationship between separable smoothness and $\ell_\infty$-smoothness.
The latter is the less restrictive, more natural assumption, and sufficient to support all existing results.
Second, by studying the corresponding smoothness constant $L_\infty$ we give a clear characterization of properties of the objective function which drive the performance of sign(S)GD and related methods, e.g., Adam.
We show that sign(S)GD may be preferable to (stochastic) gradient descent in the presence of \emph{outlier eigenvalues} ($\lambda_\text{max} \gg \bar{\lambda}$) given that the objective is somewhat axis-aligned.
Notably, these properties have independently been found to be present in neural network training tasks.
Hence, sign(S)GD need not necessarily be seen as a gradient compression scheme for distributed optimization, i.e., as an approximation of (S)GD, but it is a well-motivated optimization method in its own right.

Our approach of understanding sign(S)GD via the corresponding smoothness constant can be transferred to other methods.
In Appendix~\ref{apx:steepest_descent}, we briefly discuss the case of block-normalized gradient descent, which can be seen as a block-wise generalization of signGD.

Finally, we want to mention two other interesting aspects of sign-based methods, which we sidestepped in this work.
First, we focused on the geometric implications of sign-based methods, leaving aside the impact of gradient noise.
Since the magnitude of sign-based updates is bounded, one might hope to gain additional robustness to noise.
This aspect features somewhat implicitly in the work of~\citet{bernstein2018signsgd, bernstein2019signsgd} and \citet{safaryan2019stochastic}.
Relatedly, \citet{zhang2019adam} suggest that Adam has a certain robustness to heavy-tailed gradient noise.

Second, in practice, sign(S)GD is often used in the form of Eq.~\eqref{eq:signgd_without_norm} and not in the norm-scaled version arising in the steepest descent framework (Eq.~\ref{eq:signgd_with_norm}).
There are various possible reasons for this discrepancy, which we discuss in Appendix~\ref{apx:normalized_steepest_descent}.
In particular, we expand on the work of \citet{zhang2019analysis} who recently suggested that \emph{normalized} gradient descent is actually adapted to a certain ``relaxed smoothness'' condition.
We generalize this argument to signGD, providing a possible explanation for the aforementioned discrepancy.

\newpage
\subsubsection*{Acknowledgments}
The authors thank Hanie Sedghi and Frederik Kunstner for proofreading and helpful comments.
Lukas Balles kindly acknowledges the support of the International Max Planck Research School for Intelligent Systems (IMPRS-IS) as well as financial support by the European Research Council through ERC StG Action 757275 / PANAMA.

\bibliography{references}
\bibliographystyle{icml2020}

\newpage
\appendix

\onecolumn

\section*{---Supplementary Material---}
We now provide additional details and results for the paper. More precisely:
\begin{itemize}
\item Appendix~\ref{apx:steepest_descent} extends the discussion of steepest descent methods.
\item Appendix~\ref{apx:example_2d_quadratic} gives details on the two-dimensional quadratic example used in Section~\ref{sec:understanding_linf}.
\item Appendix~\ref{apx:normalized_steepest_descent} discusses normalized steepest descent methods as a possible explanation for the discrepancy between Eq.~\eqref{eq:signgd_without_norm} and Eq.~\eqref{eq:signgd_with_norm}.
\item Appendix~\ref{apx:experimental_details} provides details on the experiments.
\item All proofs, including those of results in the appendices, can be found in Appendix~\ref{apx:proofs}.
\end{itemize}

\section{Details on Steepest Descent}
\label{apx:steepest_descent}
In this section, we provide a more comprehensive overview of steepest descent methods including various convergence results.

\subsection{The Steepest Descent Operator}

Following earlier works on steepest descent methods \citep[e.g.][]{kelner2014almost}, it will be useful to re-write the steepest descent update (Eq.~\ref{eq:steepest_descent_def}) as
\begin{equation}
    \label{eq:def_steepest_descent_operator}
    \xx_{t+1} = \xx_t - \alpha_t \ggdop{\Vert\cdot\Vert}{\nabla f_t} \quad \text{with} \quad \ggdop{\Vert\cdot\Vert}{\zz} \in \argmax_{\yy\in\R^d} \left( \langle \zz, \yy \rangle - \frac{1}{2} \Vert \yy \Vert^2 \right) \; .
\end{equation}
The equivalence arises from substituting $\yy=-\frac{1}{\alpha_t}(\xx - \xx_t)$ in Eq.~\eqref{eq:steepest_descent_def}.
This allows to concisely express the steepest descent update direction.

\subsection{Additional Examples for Steepest Descent Methods}
\label{apx:further_ggd_examples}

We give two additional examples for steepest descent methods to demonstrate the versatility of this framework.

\subsubsection{Coordinate Descent}

Steepest descent with respect to the $L^1$-norm yields coordinate descent,
\begin{equation}
    \ggdop{\Vert\cdot\Vert_1}{\nabla f} = \vert \nabla f_{i_\text{max}} \vert \ee^{(i_\text{max})}
\end{equation}
where the selected coordinate is chosen as $i_\text{max} \in \argmax_{i\in [d]} \vert \nabla f_{t,i} \vert$ and $\ee^{(i)}$ denotes the $i$-th coordinate vector.
The corresponding smoothness assumption is
\begin{equation}
    \label{eq:l1_smoothness}
    \Vert \nabla f(\yy) - \nabla f(\xx)\Vert_\infty \leq L_1 \Vert \yy -\xx\Vert_1.
\end{equation}
In fact, this assumption implies the coordinate-wise Lipschitz smoothness assumption,
\begin{equation}
    \label{eq:coordinate_lipschitz}
    \vert \nabla f(\xx + h \ee^{(i)})_i - \nabla f(\xx)_i \vert \leq L_1 \vert h \vert \quad \forall i\in [d],
\end{equation}
which is widely-used in the literature on coordinate descent, since
\begin{equation}
    \begin{split}
        \vert \nabla f(\xx + h \ee^{(i)})_i - \nabla f(\xx)_i \vert & \leq \Vert \nabla f(\xx + h\ee^{(i)}) - \nabla f(\xx)\Vert_\infty \overset{\eqref{eq:l1_smoothness}}{\leq} L_1 \Vert \xx+h\ee^{(i)} - \xx \Vert_1 = L_1 \vert h \vert.
    \end{split}
\end{equation}


\subsubsection{Block-Normalized Gradient Descent}

Assume a block structure on $\mathbb{R}^d$ given by a partitioning $\mathcal{B}=\{B_1,\dotsc, B_b\}$ of $[d]$, with $B_k\subset [d]$, $B_k\cap B_l=\emptyset$ for $k\neq l$, and $\bigcup_k B_k = [d]$. For $B\subset [d]$, define $\xx_{B}\in\R^{\vert B\vert}$ to be the vector consisting of $(x_i)_{i\in B}$.
We can now define norms with respect to this block structure, such as
\begin{equation}
    \Vert \xx\Vert^\mathcal{B}_\infty = \max_{B\in \mathcal{B}} \Vert \xx_{B} \Vert_2 \quad \text{with dual norm} \quad
    \Vert \xx\Vert^\mathcal{B}_1 = \sum_{B\in\mathcal{B}} \Vert \xx_{B} \Vert_2.
\end{equation}
Steepest descent w.r.t.~$\Vert\cdot\Vert^\mathcal{B}_\infty$ results in \emph{block-normalized gradient descent},
\begin{equation}
    \ggdop{\Vert\cdot\Vert^\mathcal{B}_\infty}{\nabla f} = \Vert \nabla f\Vert^\mathcal{B}_1 \norm_\mathcal{B}(\nabla f), \quad \norm_\mathcal{B}(\zz) = \left( \frac{\zz_{B_1}^T}{\Vert \zz_{B_1}\Vert_2}, \dotsc, \frac{\zz_{B_b}^T}{\Vert \zz_{B_b}\Vert_2} \right)^T.
\end{equation}
This method is a block-wise equivalent of sign gradient descent, normalizing the update magnitude over blocks instead of elementwise.
Variants of this method have recently been studied empirically for neural network training~\citep{yu2017block, ginsburg2019stochastic} with the blocks corresponding to the weights and biases of individual layers.

We can analyze this method in a similar fashion as we did for signGD in the main paper.
The matrix norm determining the corresponding smoothness constant is given by
\begin{equation}
    \Vert \HH\Vert^\mathcal{B}_{\infty, 1} = \max_{\Vert \xx\Vert^\mathcal{B}_\infty \leq 1} \Vert \HH\xx\Vert^\mathcal{B}_1.
\end{equation}
We can provide an upper bound for this matrix norm analogous to Proposition~\ref{proposition:linf_hessian_norm_upper_bound}.

\begin{proposition}
\label{proposition:linf_block_hessian_norm_upper_bound}
Let $\HH\in\R^{d\times d} $ be nonzero, positive semi-definite and let a block structure be given by a partitioning $\mathcal{B}$. Then
\begin{equation}
    \Vert \HH\Vert^\mathcal{B}_{\infty, 1} \leq \rho^\mathcal{B}_\text{diag}(\HH)^{-1} \sum_{B\in\mathcal{B}} \lambda_\text{max}(\HH_{BB})
\end{equation}
with $\rho^\mathcal{B}_\text{diag}(\HH) \defas \frac{\sum_{BB^\prime} \Vert \HH_{BB^\prime}\Vert_2}{\sum_B \Vert \HH_{BB}\Vert_2}$.
\end{proposition}

The quantity $\rho^\mathcal{B}_\text{diag}(\HH)$ measures the degree of concentration of $\HH$ on its block diagonal.
Without going into details, this allows us to reason about block-normalized gradient descent in a similar way as we did for sign gradient descent in the main paper.
In particular, we can identify conditions which favor block-normalized GD.
Firstly, this will require $\rho^\mathcal{B}_\text{diag}$ to be sufficiently large, i.e., some degree of concentration of the Hessian on the diagonal blocks.
Secondly, it requires a certain structure among the eigenvalues of the blocks.
With a slight abuse of notation, denote $\lambda_B = \lambda_\text{max}(\HH_{BB})$ and note that $\lambda_\text{max}(\HH) \geq \max_B \lambda_B$.
Then block-normalized GD will be favored relative to GD if
\begin{equation}
    \max_B \lambda_B \gg \frac{1}{\vert \mathcal{B}\vert} \sum_{B\in\mathcal{B}} \lambda_B
\end{equation}
given that $\rho^\mathcal{B}_\text{diag}(\HH)$ is not too small.

\subsection{Convergence Results for Steepest Descent}
\label{apx:further_convergence_results} 

Without further assumptions, smoothness guarantees convergence to a first-order stationary point.
\begin{proposition}
\label{proposition:steepest_descent_convergence_smooth}
If $f$ is $L$-smooth w.r.t.~$\Vert \cdot \Vert$, then steepest descent (Eq.~\ref{eq:steepest_descent_def}) satisfies
\begin{equation}
    \frac{1}{T} \sum_{t=0}^{T-1} \Vert \nabla f_t \Vert_\ast^2 \leq \frac{ 2 L (f_0 - f^\star)}{T}.
\end{equation}
\end{proposition}

For smooth and convex functions,~\citet{kelner2014almost} showed $O(1/T)$ convergence in suboptimality. We restate this result here for completeness.
\begin{theorem}[Theorem 1 in~\citet{kelner2014almost}]
If $f$ is $L$-smooth w.r.t.~$\Vert \cdot \Vert$ and convex, then steepest descent (Eq.~\ref{eq:steepest_descent_def}) satisfies
\begin{equation}
    f_T - f^\star \leq \frac{2LR^2}{T+4} \quad \text{ with } \quad R \defas \max_{\xx \text{ s.t. } f(\xx) \leq f(\xx_0)} \: \min_{\xx^\star \text{ s.t. } f(\xx^\star) = f^\star} \Vert \xx - \xx^\star \Vert.
\end{equation}
\end{theorem}
The rate has a dependence on the initial distance to the nearest minimizer measured in the respective norm.

It is also straight-forward to show linear convergence in suboptimality under an additional assumption, known as the Polyak-{\L}ojasiewicz (PL) condition.
It is usually given for the Euclidean case as $\Vert \nabla f(\xx)\Vert_2^2 \geq 2\mu (f(\xx) - f^\star)$ but can likewise be formulated for arbitrary norms.
\begin{definition}
\label{def:pl}
A function $f\colon \mathbb{R}^d \to \mathbb{R}$ satisfies the PL condition with constant $\mu$ w.r.t.~a norm $\Vert \cdot \Vert$ if $\Vert \nabla f(\xx) \Vert_\ast^2 \geq 2\mu (f(\xx)- f^\star)$ for all $\xx\in\R^d$.
\end{definition}
We refer to this as PL with respect to $\Vert \cdot\Vert$ even though only the dual norm appears in the definition, since it is the natural counterpart to smoothness w.r.t.~$\Vert\cdot\Vert$ and used to prove linear convergence for steepest descent w.r.t.~$\Vert\cdot\Vert$.
As with smoothness, we have equivalence of the PL condition for all norms, but constants may differ.
Note that strong convexity implies the PL condition, but the class of PL functions also covers some non-convex functions.
\begin{proposition}
\label{proposition:steepest_descent_convergence_pl}
If $f$ is $L$-smooth and fulfills the PL condition with constant $\mu$ w.r.t~$\Vert \cdot \Vert$, then steepest descent (Eq.~\ref{eq:steepest_descent_def}) satisfies
\begin{equation}
    f_T - f^\star \leq \left( 1 - \frac{\mu}{L} \right)^T (f_0 - f^\star).
\end{equation}
\end{proposition}
We are not aware of any published work showing this simple result in its general form.

\section{Two-Dimensional Quadratic Example}
\label{apx:example_2d_quadratic}

We give details on the two-dimensional quadratic example used in Section~\ref{sec:understanding_linf}.
We consider $f(\xx) = \frac{1}{2} \xx^T \HH \xx$ with positive definite Hessian $\nabla^2 f(\xx) \equiv \HH = \left[ \begin{smallmatrix} a & b \\ b & d \end{smallmatrix}\right]$. 
In this case, we can easily find closed-form expressions to $L_\infty(\HH)$ and $L_\text{sep}(\HH)$ in Proposition~\ref{proposition:linf_lsep}.

For $L_\infty$, we have $L_\infty (\HH) = \Vert \HH\Vert_{\infty, 1} = \max_{\Vert \xx\Vert_\infty\leq 1} \Vert \HH\xx \Vert_1 = \max_{\Vert \xx\Vert_\infty\leq 1} \left( \vert ax_1 + bx_2\vert + \vert bx_1 + d x_2\vert \right)$.
Recall that $a, d>0$ by positive definiteness of $\HH$.
By separating the two cases $b\geq 0$ and $b<0$ we find
\begin{equation}
    L_\infty(\HH) = a + d + 2 \vert b\vert.
\end{equation}

Hence, $L_\infty$ and $L_\text{sep}$ coincide in this case.
Furthermore, the upper-bound in Proposition~\ref{proposition:linf_lsep} is also tight, 
\begin{equation}
    \rho_\text{diag}(\HH)^{-1} \, \sum_i \lambda_i = a + d + 2 \vert b\vert,
\end{equation}
since $\sum_i \lambda_i = \trace(\HH) = a + d$ and $\rho_\text{diag}(\HH) = (a + d) / (a + d + 2 \vert b\vert)$.

\emph{Side note} This example also reveals another problem with the separable smoothness condition.
In the literature using the separable smoothness condition it is assumed that the convergence of GD is determined by $l_\text{max}$, which implicitly assumes $l_\text{max}=\lambda_\text{max}$.
The example above shows that this may be misleading since 
The eigenvalues of $\HH$ evaluate to
\begin{equation}
    \lambda_{1/2} = \frac{a+d}{2} \pm \sqrt{ \frac{(a-d)^2}{4} + b^2}.
\end{equation}
W.l.o.g. assume $d\geq a$.
Then we have 
\begin{equation}
    l_\text{max} = d + \vert b\vert \quad \text{ and } \quad
    \lambda_\text{max} = \frac{a+d}{2} + \sqrt{ \frac{(a-d)^2}{4} + b^2}
\end{equation}
We see that $\lambda_\text{max}$ can easily exceed $l_\text{max}$.
Of course, we could deviate from the definition of Eq.~\eqref{eq:Lsep_def} and choose $l_i \equiv \lambda_\text{max}$ to guarantee $l_\text{max}=\lambda_\text{max}$.
However, this bound would very unfavorable for sign(S)GD, whose performance depends on $\sum_i l_i$.

\section{On Normalized Steepest Descent}
\label{apx:normalized_steepest_descent}

There is a discrepancy between the version of sign gradient descent arising as steepest descent w.r.t.~maximum norm, $\xx_{t+1} = \xx_t - \alpha \Vert \nabla f_t\Vert_1 \sign(\nabla f_t)$, and signSGD as used in neural network training, $\xx_{t+1} = \xx_t - \alpha_t \sign(\g_t)$, with a constant or manually decreasing step size sequence $\alpha_t$.
The norm-scaled version has actually been shown to be useful in the smooth, convex, non-stochastic setting; e.g., \citet{kelner2014almost} successfully apply it to solve max-flow problems in graphs.
Without the scaling by the gradient norm, we have a \emph{normalized} method with an update magnitude determined solely by the step size, independent of the gradient magnitude.
There are various possible reasons why this might be beneficial in settings like neural network training.

One possible explanation is that normalized methods address certain challenges of non-convex problems, e.g., by escaping from saddle points faster \citep{levy2016power} or converging to global minima for quasi-convex functions \citep{hazan2015beyond}.

Another rationale is that, in the stochastic optimization setting, a decreasing step size is needed anyway to enforce convergence.
Since $\Vert \nabla f_t\Vert_1$ is not available in the stochastic setting---and $\Vert \g_t\Vert_1$ may be a poor estimate---it might be easier to subsume a similar scaling effect in the manually-tuned step size schedule.

We want to add to this discussion in two ways.
First, we provide a basic convergence result for non-stochastic normalized steepest descent methods with a decreasing step size under the classical smoothness assumption, which we could not find in the literature.
While normalized methods are clearly suboptimal under that assumption, this provides at least a basic convergence guarantee.

Second, we expand on the work of \citet{zhang2019analysis}, who show that normalized gradient descent is adapted to a certain relaxed smoothness condition which might be a better description of the regularity exhibited by neural network training objectives, which are not smooth in the sense of Eq.~\eqref{eq:general_smoothness}.
By extending their reasoning to arbitrary norms, we provide a possible explanation for the success of normalized signGD (Eq.~\ref{eq:signgd_without_norm}).

\subsection{Convergence of Normalized Steepest Descent Under Classical Smoothness}

We show convergence to a first-order stationary point for normalized steepest descent
\begin{equation}
    \xx_{t+1} = \xx_t - \alpha_t \frac{\ggdop{\Vert\cdot\Vert}{\nabla f_t}}{\Vert \nabla f_t\Vert_\ast}
\end{equation}
with a decreasing step size schedule.
To see that this is indeed a \emph{normalized} method, recall from Lemma~\ref{lemma:dual_identities} that $\Vert \ggdop{\Vert\cdot\Vert}{\zz} \Vert = \Vert \zz \Vert_\ast$.
Normalized steepest descent w.r.t.~the maximum norm thus is sign gradient descent in the version of Eq.~\eqref{eq:signgd_without_norm}.

\begin{proposition}
\label{proposition:normalized_steepest_descent_convergence}
Let $f\colon\R^d\to\R$ be $L$-smooth w.r.t.~$\Vert\cdot\Vert$ and assume we perform normalized steepest descent updates
\begin{equation}
    \xx_{t+1} = \xx_t - \frac{\alpha_t}{L} \frac{\ggdop{\Vert\cdot\Vert}{\nabla f_t}}{\Vert \nabla f_t\Vert_\ast}
\end{equation}
with $\alpha_t = \frac{1}{\sqrt{t+1}}$.
Then
\begin{equation}
    \frac{1}{T} \sum_{t=0}^{T-1} \Vert \nabla f_t\Vert_\ast \leq \frac{L(f_0 - f_\star)}{\sqrt{T}} + \frac{\log(T+1)}{2\sqrt{T}} \, \overset{T\rightarrow\infty}{\longrightarrow}\, 0.
\end{equation}
\end{proposition}
The proof may be found in Appendix~\ref{apx:proofs}.

\subsection{Relaxed Smoothness}
\label{sec:generalizing_relaxed_smoothness}

\citet{zhang2019analysis} discuss a ``soft'' version of normalized gradient descent method,
\begin{equation}
    \label{eq:normalized_gradient_descent}
    \xx_{t+1} = \xx_t - \alpha \frac{1}{\Vert \nabla f_t\Vert_2 + \beta} \nabla f_t.
\end{equation}
and show that it is better geared towards the type of regularity exhibited by neural network training objectives, which are not smooth in the sense of Eq.~\eqref{eq:general_smoothness}.
In particular, they show that this method is optimal under a certain ``relaxed'' smoothness assumption, which allows the curvature to grow with the gradient norm instead of bounding it globally as in classical smoothness.

Since they consider normalized gradient descent, their discussion is based on Euclidean geometry. We show here that their reasoning can be generalized to arbitrary norms, providing another possible explanation of the practical success of \emph{normalized} steepest descent methods, e.g., sign gradient descent without the scaling by $\Vert\nabla f_t\Vert_1$.

\subsubsection{Results of~\texorpdfstring{\citet{zhang2019analysis}}{Zhang 2019}}

The relaxed smoothness condition proposed by \citet{zhang2019analysis} reads
\begin{equation}
    \label{eq:euclidean_relaxed_smoothness}
    \Vert \nabla^2 f(\xx) \Vert_2 \leq L^{(0)} + L^{(1)} \Vert \nabla f(\xx)\Vert_2,
\end{equation}
where $\Vert \cdot \Vert_2$ for matrices denotes the spectral norm.
This allows the curvature to grow with the gradient norm, in contrast to classical smoothness, which demands a global bound on the Hessian.

This relaxed smoothness gives rise to normalized gradient descent since, as we will see later, it provides local quadratic bounds of the form
\begin{equation}
    f_{t+1} \leq f_t + \langle \nabla f_t, \xx_{t+1} - \xx_t\rangle + \frac{1}{2} (A + B \Vert \nabla f_t\Vert_2) \Vert \xx_{t+1} - \xx_t \Vert_2^2, \quad A, B\geq 0.
\end{equation}
This resembles the bound of Lemma~\ref{lemma:general_smoothness_improvement_bound}, but the quadratic term now scales with the gradient norm.
It is minimized by a normalized gradient descent update (Eq.~\ref{eq:normalized_gradient_descent}) with appropriately chosen $\alpha$ and $\beta$.

The main finding of \citet{zhang2019analysis} is that gradient descent can become arbitrarily slow for the class of functions satisfying this relaxed smoothness, whereas normalized gradient descent (Eq.~\ref{eq:normalized_gradient_descent}) retains an  $O(1/\varepsilon^2)$ rate of convergence to an $\varepsilon$-stationary point.

\subsubsection{Generalization to Arbitrary Norms}

In this section, we generalize the concept of relaxed smoothness to arbitrary norms, which will give rise to general normalized steepest descent methods.
We define relaxed smoothness w.r.t.~to some norm $\Vert\cdot\Vert$ analogously to the Euclidean case (Eq.~\ref{eq:euclidean_relaxed_smoothness}), but use the dual norm for the gradient and the induced matrix norm of Proposition~\ref{proposition:smoothness_from_hessian_bound} for the Hessian.
\begin{definition}
A function $f$ is called $(L^{(0)}, L^{(1)})$-smooth with respect to some norm $\Vert\cdot\Vert$ if
\begin{equation}
    \label{eq:generalized_relaxed_smoothness}
    \Vert \nabla^2 f(\xx) \Vert \leq L^{(0)} + L^{(1)} \Vert \nabla f(\xx)\Vert_\ast,
\end{equation}
where $\Vert\cdot\Vert$ for matrices is the norm defined in Eq.~\eqref{eq:induced_matrix_norm}.
\end{definition}
Under this smoothness assumption, we have the following local quadratic bound:
\begin{lemma}
\label{lemma:general_relaxed_smoothness_improvement_bound}
Assume $f$ is $(L^{(0)}, L^{(1)})$-smooth with respect to a norm $\Vert \cdot\Vert$. Then for $\xx, \yy\in\mathbb{R}^d$ with $\Vert \yy-\xx\Vert \leq \frac{1}{L^{(1)}}$,
\begin{equation}
    f(\yy) \leq f(\xx) + \langle \nabla f(\xx), \yy-\xx\rangle + \frac{1}{2} (5L^{(0)} + 4L^{(1)} \Vert \nabla f(\xx)\Vert_\ast ) \Vert \yy- \xx\Vert^2.
\end{equation}
\end{lemma}
This resembles the bound in Lemma~\ref{lemma:general_smoothness_improvement_bound}, but the quadratic term now scales with the gradient norm.
In analogy to steepest descent, we can now construct an optimization method that minimizes this bound in each step.
Using Lemma~\ref{lemma:general_relaxed_smoothness_improvement_bound} with $\xx=\xx_t$ and minimizing w.r.t.~$\yy$ yields
\begin{equation}
    \label{eq:normalized_steepest_descent}
    \xx_{t+1} = \xx_t - \frac{1}{(5L^{(0)} + 4L^{(1)} \Vert \nabla f_t\Vert_\ast)} \ggdop{\Vert\cdot\Vert}{\nabla f_t}.
\end{equation}
This can be seen as a ``soft'' version of normalized steepest descent which reverts back to steepest descent in the vicinity of a stationary point, i.e., when $\Vert\nabla f_t\Vert_\ast$ becomes small.

We now show that the convergence theorem of~\citet{zhang2019analysis} carries over to this generalized setting.
The proofs (see Appendix~\ref{apx:proofs}) are straight-forward adaptations of that in~\citet{zhang2019analysis}, with a little bit of extra care with regards to the norms.
\begin{theorem}
\label{thm:convergence_generalized_normalized_gradient_descent}
Assume $f$ is $(L^{(0)}, L^{(1)})$-smooth with respect to a norm $\Vert \cdot\Vert$.
Then normalized steepest descent (Eq.~\ref{eq:normalized_steepest_descent}) converges to an $\varepsilon$-stationary point, $\Vert \nabla f\Vert_\ast \leq \varepsilon$, in at most
\begin{equation}
    T_\varepsilon = 18 (f_0 - f^\star) \max \left( \frac{L^{(0)}}{\varepsilon^2},  \frac{(L^{(1)})^2}{L^{(0)}} \right)
\end{equation}
iterations.
\end{theorem}

\section{Experimental Details}
\label{apx:experimental_details}

\subsection{Adam Experiments}

Here, we provide details for the experiments presented in Figure~\ref{fig:shuffle_Adam_mnist}.
In order to stay as close as possible to the original Adam method, we include $\varepsilon$ and consider the decomposition
\begin{equation}
    \frac{\mm_t}{\sqrt{\vv_t} + \varepsilon} = \underbrace{ \frac{\vert \mm_t\vert}{\sqrt{\vv_t} + \varepsilon} }_{=\boldsymbol{\gamma_t}} \, \sign(\mm_t).
\end{equation}
We implemented experiments using JAX \citep{jax2018github}.
For ease of implementation, shuffling and averaging is done independently for each variable (ie., weight matrix, convolution filter, or bias vector).

\paragraph{Fashion-MNIST}
We use a vanilla CNN architecture with two convolutional layers with a receptive field of $5\times 5$ pixels and $32$ and $64$ filters, respectively.
These are interspersed with max-pooling layers with a receptive field of $3\times 3$ pixels and a stride of $2$.
We add a hidden fully-connected layer with $1024$ units and an output layer with $10$ units.
ReLU activation is used for all layers except for the output layer, which uses softmax.
We use cross-entropy loss without any additional regularization and train for $6$ epochs using a mini-batch size of $64$.

\paragraph{CIFAR-10}
We use a vanilla CNN architecture with three convolutional layers with $64$ filters and a receptive field of $5\times 5$, $96$ filters and a receptive field of $3\times 3$, and $128$ filters and a receptive field of $3\times 3$, respectively, interspersed with max-pooling layers with a receptive field of $3\times 3$ pixels and a stride of $2$.
We add two hidden fully-connected layers with $512$ and $256$ units, respectively, and an output layer with $10$ units.
ReLU activation is used for all layers except for the output layer, which uses softmax.
We use cross-entropy loss without and $\ell_2$-regularization with a factor of $0.002$ and train for $10$ epochs using a mini-batch size of $128$.

\paragraph{Step size tuning}
Per problem, we manually tune one step size for momentum-SGD and one for Adam on a logarithmic grid for minimal training loss after the fixed number of epochs.
The shuffled and average variant of Adam use the same step size as original Adam.
The best performance was achieved with the following step sizes:
\begin{itemize}
    \item Fashion-MNIST: $.06$ for M-SGD, $.001$ for Adam.
    \item CIFAR-10: $.06$ for M-SGD, $.0006$ for Adam.
\end{itemize}

\subsection{Quadratic Experiments}

\paragraph{Generating Hessians.}
We draw a random rotation matrix $\boldsymbol{R}$ from the Haar distribution\footnote{The uniform distribution on the special orthogonal group $SO(d)$ of $d$-dimensional rotation matrices.
We used the \texttt{special\_ortho\_group} function provided by the \texttt{scipy.stats} package~\citep{scipy}.} and set the Hessian to be $\HH = \boldsymbol{R}^\theta \boldsymbol\Lambda (\boldsymbol{R}^\theta)^\ast$, where $\boldsymbol{R}^\theta$ for $\theta\in [0,1]$ is a non-integer matrix power and $\boldsymbol{A}^\ast$ denotes the conjugate transpose matrix.
We can think of this as rotating the eigenvectors of the Hessian by a fraction of $\theta$ in the direction prescribed by $\boldsymbol{R}$. 
The non-integer matrix power $\boldsymbol{R}^\theta$, is computed via the eigendecomposition $\boldsymbol{R}=\boldsymbol{UDU}^\ast$ as $\boldsymbol{R}^{\theta}=\boldsymbol{UD}^{\theta} \boldsymbol{U}^\ast$ where $\boldsymbol{D}^\theta$ for the diagonal matrix $\boldsymbol{D}$ is obtained by raising its elements to the power $\theta$.

\paragraph{Computing $L_\infty$.}
To compute the smoothness constant w.r.t.~the maximum norm, we have to compute the matrix norm $\Vert\HH\Vert_{\infty, 1} = \max_{\Vert\xx\Vert_\infty \leq 1} \Vert \HH\xx\Vert_1$.
We use the fact that the solution is attained at $\xx\in \{-1, 1\}^d$~\citep[see][]{rohn2000computing} and brute-force search for the maximum $\Vert \HH\xx\Vert_1$ in this set.
Since there are $2^d$ vectors in $\{-1, 1\}^d$, this is only possible for relatively small dimension.

\paragraph{On the performance measure.}
When comparing gradient descent and sign gradient descent on these quadratic problems, we use the distance to the optimum as a performance measure.
The reason is that we are interested in a comparison over a range of different quadratics with varying $\lambda_\text{max}$.
The function value, which scales with $\lambda_\text{max}$ would not be suitable for such a comparison.
Since we are comparing optimization methods which are adapted to different norms, it might make a difference which norm we choose to compute the distance to the optimum.
We opted for the Euclidean norm to benefit the baseline method (gradient descent) as the lesser of two evils.

\section{Proofs}
\label{apx:proofs}

This section contains proofs for all statements in the main text as well in the appendix.
We proceed by order of appearance.

The proofs relating to steepest descent methods make use of the formulation introduced in Eq.~\eqref{eq:def_steepest_descent_operator} with the steepest descent operator $\ggdop{\Vert\cdot\Vert}{\zz}$.
For later use, we establish the following Lemma connecting this steepest descent operator to the dual norm.
\begin{lemma}
\label{lemma:dual_identities}
For all $\xx, \zz\in\mathbb{R}^d$, we have
\begin{align}
    \text{(a)} & \phantom{=} \langle \xx, \zz \rangle \leq \Vert \xx\Vert \Vert \zz\Vert_\ast \label{eq:dual_identity_a}\\
    \text{(b)} & \phantom{=} \Vert \ggdop{\Vert\cdot\Vert}{\zz} \Vert^2 = \langle \zz, \ggdop{\Vert\cdot\Vert}{\zz} \rangle \label{eq:dual_identity_b}\\
    \text{(c)} & \phantom{=} \Vert \ggdop{\Vert \cdot\Vert}{\zz} \Vert = \Vert \zz\Vert_\ast \label{eq:dual_identity_c}
\end{align}
\end{lemma}

\begin{proof}[Proof of Lemma~\ref{lemma:dual_identities}]
Statement (a) follows immediately from the definition of the dual norm.

Regarding (b), by definition of $\ggdop{\Vert\cdot\Vert}{\zz}$, we know that $\langle \zz, c\ggdop{\Vert\cdot\Vert}{\zz} \rangle - \frac{1}{2} \Vert c \ggdop{\Vert\cdot\Vert}{\zz}\Vert^2$ is maximized by $c=1$.
Hence the derivative w.r.t.~$c$,
\begin{equation}
    \frac{d}{dc} \left[ \langle \zz, c \ggdop{\Vert\cdot\Vert}{\zz} \rangle - \frac{1}{2} \Vert c \ggdop{\Vert\cdot\Vert}{\zz} \Vert^2 \right] = \langle \zz, \ggdop{\Vert\cdot\Vert}{\zz} \rangle -c \Vert \ggdop{\Vert\cdot\Vert}{\zz} \Vert^2,
\end{equation}
must evaluate to $0$ at $c=1$, which proves (b).

For $(c)$, we use the equivalent definition
\begin{equation}
    \Vert \zz\Vert_\ast = \max_{\xx\neq 0} \frac{\langle \xx, \zz\rangle}{\Vert \xx\Vert}.
\end{equation}
Assume w.l.o.g.~that $\ggdop{\Vert\cdot\Vert}{\zz}\neq 0$. Then
\begin{equation}
    \Vert \zz\Vert_\ast = \max_{\xx\neq 0} \frac{\langle \xx, \zz\rangle}{\Vert \xx\Vert} \geq \frac{\langle \ggdop{\Vert\cdot\Vert}{\zz}, \zz\rangle}{\Vert \ggdop{\Vert\cdot\Vert}{\zz} \Vert} \overset{\text{(b)}}{=} \Vert \ggdop{\Vert\cdot\Vert}{\zz} \Vert.
\end{equation}
Conversely,
Let $\xx^\prime \in \argmax_{\Vert \xx\Vert=1} \langle \xx, \zz\rangle$, such that $\langle \zz, \xx^\prime \rangle = \Vert\zz\Vert_\ast$. Then
\begin{equation}
    \frac{1}{2} \Vert \zz\Vert_\ast^2 = \langle \zz, \Vert \zz\Vert_\ast \xx^\prime \rangle - \frac{1}{2} \Vert \Vert \zz \Vert_\ast \xx^\prime \Vert^2 \leq \langle \zz, \ggdop{\Vert\cdot\Vert}{\zz} \rangle - \frac{1}{2} \Vert \ggdop{\Vert\cdot\Vert}{\zz} \Vert^2 \overset{\text{(b)}}{=} \frac{1}{2} \Vert \ggdop{\Vert\cdot\Vert}{\zz} \Vert^2,
\end{equation}
where the inequality is by definition of $\ggdop{\Vert\cdot\Vert}{\zz}$.
\end{proof}

\subsection{Proofs for Section~\ref{sec:unifying}}

\label{sec:proof_lemma_smooth}
\begin{proof}{Proof of Lemma~\ref{lemma:general_smoothness_improvement_bound}}
Define $g(\tau) = f(\xx + \tau (\yy- \xx))$ for $\tau\in [0,1]$ with $g^\prime(\tau) = \langle \nabla f(\xx+ \tau(\yy-\xx)), \yy-\xx\rangle$.
Then
\begin{equation}
\begin{split}
    f(\yy) & = f(\xx) + \int_0^1 g^\prime(\tau) d\tau \\
    & = f(\xx) + \int_0^1 \langle \nabla f(\xx+ \tau(\yy-\xx)), \yy-\xx \rangle d\tau \\
    & = f(\xx) + \langle \nabla f(\xx), \yy-\xx \rangle + \int_0^1 \langle \nabla f(\xx+ \tau(\yy-\xx)) - \nabla f(\xx), \yy-\xx \rangle d\tau\\
    & \overset{\eqref{eq:dual_identity_a}}{\leq} f(\xx) + \langle \nabla f(\xx), \yy-\xx \rangle + \int_0^1 \Vert \nabla f(\xx+ \tau(\yy-\xx)) - \nabla f(\xx) \Vert_\ast  \Vert \yy-\xx \Vert d\tau\\
    & \overset{\eqref{eq:general_smoothness}}{\leq} f(\xx) + \langle \nabla f(\xx), \yy-\xx \rangle + \int_0^1 L \Vert \tau(\yy-\xx)\Vert \Vert \yy-\xx \Vert d\tau \\
    & = f(\xx) + \langle \nabla f(\xx), \yy-\xx \rangle + L \Vert \yy-\xx \Vert^2 \int_0^1 \tau d\tau \\
    & = f(\xx) + \langle \nabla f(\xx), \yy-\xx \rangle + \frac{L}{2} \Vert \yy-\xx \Vert^2.
\end{split}
\end{equation}
The first inequality is due to Lemma \ref{lemma:dual_identities}(a) and the second inequality uses the $L$-smoothness.
\end{proof}

\begin{proof}[Proof of Lemma~\ref{lemma:steepest_descent_improvement}]
We apply Lemma~\ref{lemma:general_smoothness_improvement_bound} with $\yy=\xx^+ = \xx - \frac{1}{L} \ggdop{\Vert\cdot\Vert}{\nabla f(\xx)}$
\begin{equation}
    \begin{split}
    f(\xx^+) & \leq f(\xx) + \langle \nabla f(\xx), \xx^+ - \xx \rangle + \frac{L}{2} \Vert \xx^+ - \xx\Vert^2 \\
    & = f(\xx) + \langle \nabla f(\xx), - \frac{1}{L} \ggdop{\Vert\cdot\Vert}{\nabla f(\xx)} \rangle + \frac{L}{2} \left\Vert -\frac{1}{L} \ggdop{\Vert\cdot\Vert}{\nabla f(\xx)} \right\Vert^2 \\
    & = f(\xx) - \frac{1}{L} \left( \langle \nabla f(\xx), \ggdop{\Vert\cdot\Vert}{\nabla f(\xx)} \rangle - \frac{1}{2} \Vert \ggdop{\Vert\cdot\Vert}{\nabla f(\xx)} \Vert^2  \right).
    \end{split}
\end{equation}
By Lemma~\ref{lemma:dual_identities}, we have $\langle \nabla f(\xx), \ggdop{\Vert\cdot\Vert}{\nabla f(\xx)} \rangle = \Vert \ggdop{\Vert\cdot\Vert}{\nabla f(\xx)} \Vert^2 = \Vert \nabla f(\xx)\Vert_\ast^2$.
Substituting this in yields the desired bound.
\end{proof}

\begin{proof}[Proof of Proposition~\ref{proposition:separable_smoothness_implies_linf}]
The dual norm of $\Vert\cdot\Vert_{\LL}$ is $\Vert\cdot\Vert_{\LL^{-1}}$, such that the assumption of $1$-smoothness w.r.t.~$\Vert\cdot\Vert_{\LL}$ amounts to
\begin{equation}
    \label{eq:separable_smoothness_implies_linf_proof_1}
    \Vert \nabla f(\yy) - \nabla f(\xx) \Vert_{\LL^{-1}} \leq \Vert \yy-\xx\Vert_{\LL} \quad \forall\, \xx, \yy\in\R^d.
\end{equation}
First, by definition of the maximum norm, we get
\begin{equation}
    \label{eq:separable_smoothness_implies_linf_proof_2}
    \Vert \zz \Vert_{\LL} = \sqrt{\sum_i l_i z_i^2} \leq \sqrt{\sum_i l_i \Vert \zz\Vert_\infty^2} = \sqrt{\sum_i l_i} \Vert \zz\Vert_\infty.
\end{equation}
Secondly, using Cauchy-Schwarz,
\begin{equation}
    \label{eq:separable_smoothness_implies_linf_proof_3}
    \Vert \zz \Vert_1 = \sum_i \vert z_i\vert = \sum_i \frac{\vert z_i\vert}{\sqrt{l_i}} \sqrt{l_i} \leq \sqrt{ \sum_i \frac{z_i^2}{l_i} } \sqrt{\sum_i l_i} = \sqrt{\sum_i l_i} \Vert \zz\Vert_{\LL^{-1}}.
\end{equation}
Combining these two inequalities with the assumption yields the assertion:
\begin{equation}
\begin{split}
    \Vert \nabla f(\yy) - \nabla f(\xx)\Vert_1
    & \overset{\eqref{eq:separable_smoothness_implies_linf_proof_3}}{\leq}
    \sqrt{\sum_i l_i} \Vert \nabla f(\yy) - \nabla f(\xx)\Vert_{\LL^{-1}}
    \overset{\eqref{eq:separable_smoothness_implies_linf_proof_1}}{\leq}
    \sqrt{\sum_i l_i} \Vert \yy-\xx\Vert_{\LL} \\
    & \overset{\eqref{eq:separable_smoothness_implies_linf_proof_2}}{\leq}
    \sqrt{\sum_i l_i} \sqrt{\sum_i l_i} \Vert \yy-\xx\Vert_\infty = \left(\sum_i l_i \right) \Vert \yy-\xx\Vert_\infty.
\end{split}
\end{equation}
\end{proof}

\begin{proof}[Proof of Proposition~\ref{proposition:linf_smoothness_replaces_separable_smoothness}]
For separable smoothness, we use the definition (Eq.~\ref{eq:separable_smoothness_def}) with $\yy=\xx+\alpha \ss$ and get
\begin{equation}
    f(\xx + \alpha \ss) \leq f(\xx) + \langle \nabla f(\xx), \ss \rangle + \frac{\alpha^2}{2} \sum_i l_i \ss_i^2 = f(\xx) + \langle \nabla f(\xx), \ss \rangle + \frac{\alpha^2}{2} \sum_i l_i.
\end{equation}
due to $\ss_i^2 = 1$ for all $i$.

For $\ell_\infty$-smoothness, we use Lemma~\ref{lemma:general_smoothness_improvement_bound} with $\yy=\xx+\alpha \ss$ and get
\begin{equation}
    f(\xx + \alpha \ss) \leq f(\xx) + \langle \nabla f(\xx), \ss \rangle + \frac{L_\infty}{2} \alpha^2 \Vert \ss\Vert_\infty^2 = f(\xx) + \langle \nabla f(\xx), \ss \rangle + \frac{\alpha^2}{2} \sum_i l_i.
\end{equation}
where the second step is due to $\Vert \ss\Vert_\infty = 1$ and the assumption that $L_\infty = \sum_i l_i$.
\end{proof}

\subsection{Proofs for Section~\ref{sec:understanding_linf}}
\label{sec:understanding_linf_proofs}

\begin{proof}[Proof of Proposition~\ref{proposition:smoothness_from_hessian_bound}]
Note that the matrix norm by construction satisfies
\begin{equation}
    \label{eq:induced_matrix_norm_submultiplicative}
    \Vert \HH\xx\Vert_\ast \leq \Vert \HH\Vert \Vert \xx\Vert.
\end{equation}
We first show that Eq.~\eqref{eq:induced_matrix_norm} defines a matrix norm.
Clearly $\Vert \HH\Vert \geq 0$ and $\Vert \HH\Vert=0$ iff $\HH=0$. Furthermore, $\Vert \lambda \HH\Vert = \vert \lambda\vert \Vert \HH\Vert$.
It remains to show subadditivity. Let $\HH, \HH^\prime\in\R^{d\times d}$
\begin{equation}
    \begin{split}
        \Vert \HH + \HH^\prime \Vert & = \max_{\Vert \xx \Vert\leq 1} \Vert (\HH + \HH^\prime) \xx\Vert_\ast \leq \max_{\Vert \xx\Vert \leq 1} \left( \Vert \HH\xx\Vert_\ast + \Vert \HH^\prime \xx\Vert_\ast \right) \\
        &\leq \max_{\Vert \xx\Vert \leq1} \Vert \HH\xx\Vert_\ast + \max_{\Vert \xx^\prime \Vert\leq 1} \Vert \HH^\prime \xx^\prime \Vert_\ast = \Vert \HH\Vert + \Vert \HH^\prime\Vert.
    \end{split}
\end{equation}

Now assume $\Vert \nabla^2 f(\xx)\Vert \leq L$ for all $x\in\R^d$.
Let $\xx, \yy\in\mathbb{R}^d$ and define $g(\tau)= \nabla f(\xx + \tau(\yy-\xx))$ for $\tau\in [0,1]$.
\begin{equation}
    \begin{split}
        \Vert \nabla f(\yy) - \nabla f(\xx) \Vert_\ast & = \left\Vert \int_0^1 \frac{d}{d\tau} g(\tau) d\tau \right\Vert_\ast \\
        & = \left\Vert \int_0^1 \nabla^2 f(\xx+ \tau(\yy-\xx)) (\yy-\xx) d\tau \right\Vert_\ast \\
        & \leq \int_0^1 \left\Vert \nabla^2 f(\xx+ \tau(\yy-\xx)) (\yy-\xx) \right\Vert_\ast d\tau\\
        & \overset{\eqref{eq:induced_matrix_norm_submultiplicative}}{\leq} \int_0^1 \Vert \nabla^2 f(\xx + \tau(\yy-\xx)) \Vert \Vert \yy-\xx \Vert d\tau \\
        & \leq L \Vert \yy- \xx\Vert \int_0^1 1 d\tau = L \Vert \yy-\xx\Vert.
    \end{split}
\end{equation}

Conversely, assume $L$-smoothness and fix $\xx\in\R^d$.
For any $\Vert \ss \Vert\leq 1$ and $\varepsilon>0$,
\begin{equation}
        \left\Vert \left( \int_0^\varepsilon \nabla^2 f(\xx + \tau \ss) d\tau \right) \ss \right\Vert_\ast = \Vert \nabla f(\xx+ \varepsilon \ss) - \nabla f(\xx) \Vert_\ast \leq \varepsilon L \Vert \ss \Vert \leq \varepsilon L.
\end{equation}
Dividing by $\varepsilon$ and letting $\varepsilon \rightarrow 0$, we get
\begin{equation}
\begin{split}
    \Vert \nabla^2 f(\xx) \: \ss \Vert_\ast &= \left\Vert  \lim_{\varepsilon\rightarrow 0} \left( \frac{1}{\varepsilon} \int_0^\varepsilon \nabla^2 f(\xx + \tau \ss) d\tau \right) \ss \right\Vert_\ast \\
    & =  \lim_{\varepsilon\rightarrow 0} \frac{1}{\varepsilon} \left\Vert \left( \int_0^\varepsilon \nabla^2 f(\xx + \tau \ss) d\tau \right) \ss \right\Vert_\ast \leq L.
\end{split}
\end{equation}
This implies $\Vert\nabla^2 f(\xx)\Vert = \sup_{\Vert \ss\Vert\leq 1} \Vert \nabla^2 f(\xx) \ss \Vert_\ast \leq L$.
\end{proof}

\begin{proof}[Proof of Proposition~\ref{proposition:linf_hessian_norm_upper_bound}]
First note that
\begin{equation}
    \Vert \HH \Vert_{\infty, 1} \defas \max_{\Vert \xx\Vert_\infty\leq 1} \Vert \HH\xx \Vert_1 = \max_{\Vert \xx\Vert_\infty\leq 1} \sum_i \left\vert \sum_j H_{ij} x_j \right\vert \leq \sum_{i, j} \vert H_{ij} \vert.
\end{equation}
Recall that $\sum_i \vert H_{ii}\vert = \sum_i H_{ii} = \sum_i \lambda_i$ for positive definite matrices. Then
\begin{equation}
    \Vert \HH \Vert_{\infty, 1} \leq \sum_{i, j} \vert H_{ij}\vert = \frac{\sum_{i, j}\vert H_{ij}\vert}{\sum_i \vert H_{ii} \vert} \sum_i \lambda_i = \rho_\text{diag}(\HH)^{-1} \sum_i \lambda_i.
\end{equation}
\end{proof}

\begin{proof}[Proof of Proposition~\ref{proposition:linf_hessian_norm_upper_bound_extended}]
Since $\HH$ is a real symmetric matrix, it has a system of orthonormal eigenvectors and can be written as
\begin{equation}
    \HH = \sum_i \lambda_i \vv^{(i)} (\vv^{(i)})^T.
\end{equation}
With that we find
\begin{equation}
    \begin{split}
        \Vert \HH\Vert_{\infty, 1} & = \max_{\Vert\xx\Vert_\infty \leq 1} \Vert \HH\xx\Vert_1 = \max_{\Vert\xx\Vert_\infty \leq 1} \left\Vert \sum_i \lambda_i (\xx^T \vv^{(i)}) \vv^{(i)} \right\Vert_1 \\
        & \leq \max_{\Vert\xx\Vert_\infty \leq 1} \sum_i \vert \lambda_i\vert \; \underbrace{ \vert \xx^T \vv^{(i)} \vert}_{\leq \Vert \vv^{(i)}\Vert_1} \; \Vert \vv^{(i)} \Vert_1 \leq \sum_i \vert \lambda_i\vert \; \Vert v^{(i)}\Vert_1^2.
    \end{split}
\end{equation}
\end{proof}

\begin{proof}[Proof of \texorpdfstring{Eq.~\eqref{eq:general_bound_linf_l2}}{the general bound}]
The fact that $\Vert \zz\Vert_\infty \leq \Vert \zz\Vert_2 \leq \Vert \zz\Vert_1$ implies
\begin{equation}
    L_2 = \sup_{\xx \neq \yy} \frac{\Vert \nabla f(\yy) - \nabla f(\xx)\Vert_2}{\Vert \yy-\xx\Vert_2} \leq \sup_{\xx\neq \yy} \frac{\Vert \nabla f(\yy) - \nabla f(\xx)\Vert_1}{\Vert \yy-\xx\Vert_\infty} = L_\infty.
\end{equation}
Conversely, using $\frac{1}{\sqrt{d}} \Vert \zz\Vert_1 \leq \Vert \zz\Vert_2 \leq \sqrt{d} \Vert \zz\Vert_\infty$ 
\begin{equation}
    L_\infty = \sup_{\xx\neq \yy} \frac{\Vert \nabla f(\yy) - \nabla f(\xx)\Vert_1}{\Vert \yy-\xx\Vert_\infty} \leq \sup_{\xx\neq\yy} \frac{\sqrt{d} \Vert \nabla f(\yy) - \nabla f(\xx)\Vert_2}{\frac{1}{\sqrt{d}}\Vert \yy-\xx\Vert_2} = dL_2.
\end{equation}
\end{proof}

\begin{proof}[Proof of Proposition~\ref{proposition:linf_lsep}]
First inequality:
First, let $\hat{l}_1,\dotsc, \hat{l}_d \geq 0$ be the minimizer in the definition of $L_\text{sep}$. For any $\zz$ with $\Vert \zz\Vert_\infty \leq 1$, we have
\begin{equation}
    \label{eq:thm_linf_lsep_eq3}
    \zz^T\HH \zz \leq \zz^T \diag(\hat{l}_1, \dotsc, \hat{l}_d) \zz = \sum_i \hat{l}_i z_i^2 \leq \sum_i \hat{l}_i.
\end{equation}
Next, we rewrite the definition of $L_\infty$ as
\begin{equation}
    L_\infty = \max_{\Vert \xx\Vert_\infty, \Vert \yy \Vert_\infty\leq 1} \xx^T \HH \yy.
\end{equation}
and let $(\hat{\xx}, \hat{\yy})$ be the maximizer.
Then due to $\HH$ being psd, we have
\begin{equation}
    \begin{split}
        0 & \leq (\hat{\xx} - \hat{\yy})^T \HH (\hat{\xx} - \hat{\yy}) =  \hat{\xx}^T \HH \hat{\xx} - 2\hat{\xx}^T \HH \hat{\yy} + \hat{\yy}^T\HH \hat{\yy} \leq 2 \sum_i \hat{l}_i - 2 \hat{\xx}^T \HH \hat{\yy},
    \end{split}
\end{equation}
where the last inequality is due to Eq.~\eqref{eq:thm_linf_lsep_eq3} applied to $\hat{\xx}$ and $\hat{\yy}$.
This proves the assertion, since $\sum_i \hat{l}_i = L_\text{sep}$ and $\hat{\xx}^T \HH \hat{\yy} = L_\infty$ by definition.

Second inequality:
We set $l_i= \sum_{j} \vert H_{ij} \vert$ and denote $\LL=\diag(l_1, \dotsc, l_d)$. Then
\begin{gather}
    [\LL - \HH]_{ii} \geq \sum_{j\neq i} \vert H_{ij}\vert \\
    \sum_{j\neq i} \vert [\LL - \HH ]_{ij} \vert = \sum_{j\neq i} \vert H_{ij} \vert
\end{gather}
making $\LL-\HH$ diagonally dominant with non-negative diagonal elements, hence positive semi-definite.
Therefore, $\LL$ is admissible in the definition of $L_\text{sep}$.
Now,
\begin{equation}
    \sum_i l_i =  \sum_{i, j} \vert H_{ij} \vert = \left( \frac{\sum_i \vert H_{ii}\vert }{\sum_{i, j} \vert H_{ij} \vert} \right)^{-1} \sum_i \lambda_i,
\end{equation}
where we used $\sum_i \vert H_{ii}\vert = \sum_i H_{ii} = \sum_i \lambda_i$.
\end{proof}

\subsection{Proofs for Appendix~\ref{apx:steepest_descent}}

\begin{proof}[Proof of Proposition~\ref{proposition:linf_block_hessian_norm_upper_bound}]
First note that
\begin{equation}
    \begin{split}
        \Vert \HH \Vert^\mathcal{B}_{\infty, 1} & \defas \max_{\Vert \xx\Vert^\mathcal{B}_\infty\leq 1} \Vert \HH\xx \Vert^\mathcal{B}_1 = \max_{\Vert \xx\Vert^\mathcal{B}_\infty \leq 1} \sum_{B\in\mathcal{B}} \Vert [\HH\xx]_B\Vert_2 = \max_{\Vert \xx\Vert^\mathcal{B}_\infty \leq 1} \sum_{B\in\mathcal{B}} \left\Vert \sum_{B^\prime\in\mathcal{B}} \HH_{BB^\prime} \xx_{B^\prime} \right\Vert_2\\
        & \leq \max_{\Vert \xx\Vert^\mathcal{B}_\infty \leq 1} \sum_{B, B^\prime \in\mathcal{B}} \left\Vert \HH_{BB^\prime} \xx_{B^\prime} \right\Vert_2 \leq \sum_{B, B^\prime \in\mathcal{B}} \max_{\zz\in \R^{\vert B^\prime\vert}, \Vert \zz \Vert_2 \leq 1} \left\Vert \HH_{BB^\prime} \zz \right\Vert_2 = \sum_{B, B^\prime \in\mathcal{B}} \Vert \HH_{BB^\prime} \Vert_2
    \end{split}
\end{equation}
Recall that the diagonal blocks $\HH_{BB}$ are positive definite since $\HH$ is positive definite and thus $\Vert \HH_{BB}\Vert_2 = \lambda_\text{max}(\HH_{BB})$ and
\begin{equation}
    \Vert \HH \Vert^\mathcal{B}_{\infty, 1} \leq \sum_{BB^\prime} \Vert \HH_{BB^\prime}\Vert_2 = \frac{\sum_{BB^\prime} \Vert \HH_{BB^\prime}\Vert_2}{\sum_B \Vert \HH_{BB}\Vert_2} \sum_B \lambda_\text{max}(\HH_{BB}) = \rho^\mathcal{B}_\text{diag}(\HH)^{-1} \sum_B \lambda_\text{max}(\HH_{BB}).
\end{equation}

\end{proof}

\begin{proof}[Proof of Proposition \ref{proposition:steepest_descent_convergence_smooth}]
Lemma~\ref{lemma:steepest_descent_improvement} gives
\begin{equation}
    \Vert \nabla f_t \Vert_\ast^2 \leq  2L(f_t - f_{t+1})
\end{equation}
Rearranging and summing for $t=0,\dotsc, T-1$ yields
\begin{equation}
\begin{split}
    \frac{1}{T} \sum_{t=0}^{T-1} \Vert \nabla f_t \Vert_\ast^2 & \leq \frac{2 L}{T} \sum_{t=0}^{T-1} (f_t - f_{t+1}) = \frac{2 L (f_0 - f_T)}{T}  \leq  \frac{ 2 L (f_0 - f^\star)}{T}.
\end{split}
\end{equation}
\end{proof}

\begin{proof}[Proof of Proposition~\ref{proposition:steepest_descent_convergence_pl}]
Combining Lemma~\ref{lemma:steepest_descent_improvement} and the PL condition gives
\begin{equation}
        f_{t+1} \leq f_t - \frac{1}{2L} \Vert \nabla f_t \Vert_\ast^2 \leq f_t - \frac{\mu}{L} (f_t - f^\star).
\end{equation}
Subtracting $f^\star$ from both sides and iterating backwards yields the statement.
\end{proof}

\subsection{Proofs for Appendix~\ref{apx:normalized_steepest_descent}}

We first prove Proposition~\ref{proposition:normalized_steepest_descent_convergence}.
\begin{proof}[Proof of Proposition~\ref{proposition:normalized_steepest_descent_convergence}]
By Lemma~\ref{lemma:general_smoothness_improvement_bound} we have
\begin{equation}
    f_{t+1} \leq f_t - \frac{\alpha_t}{L \Vert\nabla f_t\Vert_\ast} \langle \nabla f_t, \ggdop{\Vert\cdot\Vert}{\nabla f_t} \rangle + \frac{\alpha_t^2}{2L\Vert\nabla f_t\Vert_\ast^2} \Vert \ggdop{\Vert\cdot\Vert}{\nabla f_t} \Vert^2
\end{equation}
By Lemma~\ref{lemma:dual_identities}, we have $\langle \nabla f_t, \ggdop{\Vert\cdot\Vert}{\nabla f_t}\rangle = \Vert \ggdop{\Vert\cdot\Vert}{\nabla f_t}\Vert^2 = \Vert \nabla f_t\Vert_\ast^2$, which yields
\begin{equation}
    f_{t+1} \leq f_t - \frac{1}{L} \left( \alpha_t \Vert \nabla f_t\Vert_\ast - \frac{\alpha_t^2}{2} \right).
\end{equation}
With a telescopic sum, we get
\begin{equation}
    f_0 - f^\star \geq f_0 - f_T = \sum_{t=0}^{T-1} (f_t - f_{t+1}) \geq \frac{1}{L} \sum_{t=0}^{T-1} \alpha_t \Vert \nabla f_t\Vert_\ast - \frac{1}{2L} \sum_{t=0}^{T-1} \alpha_t^2
\end{equation}
Now with $\alpha_t = 1/\sqrt{t+1} \geq 1/\sqrt{T}$, we have
\begin{equation}
    L(f_0 - f^\star) \geq \frac{1}{\sqrt{T}} \sum_{t=0}^{T-1} \Vert \nabla f_t\Vert_\ast - \frac{1}{2} \underbrace{ \sum_{t=0}^{T-1} \frac{1}{t+1} }_{\leq \log(T+1)}
\end{equation}
and thus
\begin{equation}
    \frac{1}{T} \sum_{t=0}^{T-1} \Vert \nabla f_t\Vert_\ast \leq \frac{L(f_0 - f^\star)}{\sqrt{T}} + \frac{\log(T+1)}{2\sqrt{T}}
\end{equation}
\end{proof}

We now proceed to the proofs for Subsection~\ref{sec:generalizing_relaxed_smoothness} about the relaxed smoothness assumption.
All proofs are closely following the ones given for the Euclidean norm in~\citet{zhang2019analysis}.

To prove Lemma~\ref{lemma:general_relaxed_smoothness_improvement_bound}, we first need the following Lemma, which allows us to control the growth of the gradient norm in the vicinity of a point $\xx\in\R^d$.
\begin{lemma}
\label{lemma:general_relaxed_smoothness_control_change_in_gradient_norm}
Assume Eq.~\eqref{eq:generalized_relaxed_smoothness} holds and let $\xx, \yy$ with $\Vert \yy-\xx\Vert\leq \frac{1}{L^{(1)}}$.
Then 
\begin{equation}
    \Vert \nabla f(\yy)\Vert_\ast \leq 4 \left( \frac{L^{(0)}}{L^{(1)}} + \Vert \nabla f(\xx) \Vert_\ast \right).
\end{equation}
\end{lemma}
\begin{proof}
Define $\xx(\tau) = \xx + \tau (\yy-\xx)$ as well as $\boldsymbol{g}(\tau) = \nabla f(\xx(\tau))$ with $\boldsymbol{g}^\prime(\tau) = \nabla^2f(\xx(\tau))(\yy-\xx)$.
Then
\begin{equation}
    \begin{split}
        \Vert \nabla f(\xx(t)) \Vert_\ast & = \left\Vert \nabla f(\xx) + \int_0^t \boldsymbol{g}^\prime(\tau)d\tau \right\Vert_\ast \\
        & \leq \Vert \nabla f(\xx) \Vert_\ast + \int_0^t  \Vert \boldsymbol{g}^\prime(\tau) \Vert_\ast d\tau \\
        & = \Vert \nabla f(\xx) \Vert_\ast + \int_0^t \Vert \nabla^2 f(\xx(\tau)) (\yy-\xx) \Vert_\ast d\tau  \\
        & \overset{\eqref{eq:induced_matrix_norm_submultiplicative}}{\leq} \Vert \nabla f(\xx) \Vert_\ast + \underbrace{\Vert  (\yy-\xx) \Vert}_{\leq 1/L^{(1)}} \int_0^t \underbrace{\Vert \nabla^2 f(\xx(\tau))\Vert}_{\leq L^{(0)} + L^{(1)}\Vert \nabla f(\xx(\tau))\Vert_\ast \: \eqref{eq:generalized_relaxed_smoothness}} d\tau \\
        & \leq \Vert \nabla f(\xx)\Vert_\ast + \frac{1}{L^{(1)}} \int_0^t L^{(0)} + L^{(1)} \Vert \nabla f(\xx(\tau))\Vert_\ast d\tau \\
        & = \Vert \nabla f(\xx) \Vert_\ast + t \frac{L^{(0)}}{L^{(1)}} +  \int_0^t \Vert \nabla f(\xx(\tau))\Vert_\ast d\tau
    \end{split}
\end{equation}
Applying the integral form of Groenwall's inequality\footnote{
Groenwall's inequality says that if $u(t) \leq \alpha(t) + \int_{t_0}^t \beta(\tau) u(\tau) d\tau$ for continuous $u$ and $\beta$,
then
\begin{equation}
    u(t) \leq \alpha(t) + \int_{t_0}^t \alpha(\tau) \beta(\tau) \exp \left(\int_\tau^t \beta(r) dr \right) d\tau.
\end{equation}
We apply it here with $u(t) = \Vert \nabla f(\xx(t)) \Vert_\ast$ and $\alpha(t) = \Vert \nabla f(\xx)\Vert_\ast + t L^{(0)}/L^{(1)}$ and $\beta(\tau)\equiv 1$.
} yields
\begin{equation}
    \Vert \nabla f(\xx(t))\Vert_\ast \leq \Vert \nabla f(\xx) \Vert_\ast + t \frac{L^{(0)}}{L^{(1)}} + \int_0^t \left( \Vert \nabla f(\xx) \Vert_\ast + \tau \frac{L^{(0)}}{L^{(1)}} \right) \exp(t-\tau) d\tau.
\end{equation}
We now specialize to $t=1$ and upper-bound the integrand
\begin{equation}
    \begin{split}
        \Vert \nabla f(\yy)\Vert_\ast & = \Vert \nabla f(\xx(1))\Vert_\ast \\
        & \leq \Vert \nabla f(\xx) \Vert_\ast + \frac{L^{(0)}}{L^{(1)}} + \int_0^1 \left( \Vert \nabla f(\xx) \Vert_\ast + \underbrace{\tau}_{\leq 1} \frac{L^{(0)}}{L^{(1)}} \right) \underbrace{\exp(1-\tau)}_{\leq \exp(1) < 3} d\tau \\
        & \leq \Vert \nabla f(\xx) \Vert_\ast + \frac{L^{(0)}}{L^{(1)}} + 3 \left( \Vert \nabla f(\xx) \Vert_\ast + \frac{L^{(0)}}{L^{(1)}} \right) \int_0^1 d\tau\\
        & = 4 \left( \frac{L^{(0)}}{L^{(1)}} + \Vert \nabla f(\xx) \Vert_\ast \right).
    \end{split}
\end{equation}
\end{proof}

We can now approach the proof of Lemma~\ref{lemma:general_relaxed_smoothness_improvement_bound}.
\begin{proof}[Proof of Lemma~\ref{lemma:general_relaxed_smoothness_improvement_bound}]
According to Taylor's theorem we have
\begin{equation}
    \label{eq:general_relaxed_smoothness_improvement_bound_eq_1}
    f(\yy) = f(\xx) + \langle \nabla f(\xx), \yy-\xx \rangle + \frac{1}{2} \langle  \yy-\xx, \nabla^2 f(\boldsymbol{\xi})  (\yy-\xx) \rangle
\end{equation}
with some $\boldsymbol{\xi} \in \{ \xx + \tau (\yy-\xx) \mid \tau\in [0,1]\}$.
We can bound the quadratic term as
\begin{equation}
    \label{eq:general_relaxed_smoothness_improvement_bound_eq_2}
    \begin{split}
        \langle \yy-\xx, \nabla^2 f(\boldsymbol{\xi}) (\yy-\xx) \rangle & \overset{\eqref{eq:dual_identity_a}}{\leq} \Vert \yy-\xx\Vert \Vert \nabla^2 f(\boldsymbol{\xi}) (\yy-\xx) \Vert_\ast \\
        & \overset{\eqref{eq:induced_matrix_norm_submultiplicative}}{\leq} \Vert \yy-\xx \Vert^2 \Vert \nabla^2 f(\boldsymbol{\xi})\Vert \\
        & \overset{\eqref{eq:generalized_relaxed_smoothness}}{\leq} (L^{(0)} + L^{(1)}\Vert \nabla f(\boldsymbol{\xi})\Vert_\ast) \Vert \yy-\xx \Vert^2.
    \end{split}
\end{equation}
The first inequality is by definition of the dual norm (see Lemma~\ref{lemma:dual_identities}). The second inequality is by construction of the induced matrix norm.
The final inequality uses the relaxed smoothness assumption (Eq.~\ref{eq:generalized_relaxed_smoothness}).

Next, since $\Vert \yy-\xx\Vert\leq 1/L^{(1)}$ by assumption of Lemma~\ref{lemma:general_relaxed_smoothness_improvement_bound}, we know $\Vert \boldsymbol{\xi} -\xx \Vert\leq \frac{1}{L^{(1)}}$. Lemma~\ref{lemma:general_relaxed_smoothness_control_change_in_gradient_norm} thus gives us
\begin{equation}
    \Vert \nabla f(\boldsymbol\xi)\Vert_\ast \leq 4 \left( \frac{L^{(0)}}{L^{(1)}} + \Vert \nabla f(\xx) \Vert_\ast \right).
\end{equation}
Plugging this back into Eq.~\eqref{eq:general_relaxed_smoothness_improvement_bound_eq_2} yields
\begin{equation}
    \langle \yy-\xx, \nabla^2 f(\boldsymbol\xi) (\yy-\xx)\rangle \leq (5L^{(0)} + 4L^{(1)} \Vert \nabla f(\xx)\Vert_\ast) \Vert \yy-\xx\Vert^2.
\end{equation}
Using that in Eq.~\eqref{eq:general_relaxed_smoothness_improvement_bound_eq_1} proves the assertion.
\end{proof}

Finally, we prove Theorem~\ref{thm:convergence_generalized_normalized_gradient_descent}

\begin{proof}[Proof of Theorem \ref{thm:convergence_generalized_normalized_gradient_descent}]
Using Lemma~\ref{lemma:general_relaxed_smoothness_improvement_bound} with $\xx=\xx_t$ and $\yy=\xx_{t+1}=\xx_t - \eta_t \ggdop{\Vert\cdot\Vert}{\nabla f_t}$ yields
\begin{equation}
    f_{t+1} \leq f_t  - \eta_t  \langle \nabla f_t, \ggdop{\Vert\cdot\Vert}{\nabla f_t} \rangle + \frac{\eta_t^2}{2} (5L^{(0)} + 4L^{(1)} \Vert \nabla f_t\Vert_\ast) \Vert \ggdop{\Vert\cdot\Vert}{\nabla f_t} \Vert^2.
\end{equation}
Recall from Lemma~\ref{lemma:dual_identities} that $\langle z, \ggdop{\Vert\cdot\Vert}{z} \rangle = \Vert \ggdop{\Vert\cdot\Vert}{z} \Vert^2 = \Vert z\Vert_\ast^2$ and, hence,
\begin{equation}
\begin{split}
    f_{t+1} & \leq f_t  - \eta_t  \langle \nabla f_t, \ggdop{\Vert\cdot\Vert}{\nabla f_t} \rangle + \frac{\eta_t^2}{2} (5L^{(0)} + 4L^{(1)} \Vert \nabla f_t\Vert_\ast) \Vert \ggdop{\Vert\cdot\Vert}{\nabla f_t} \Vert^2 \\
    & = f_t - \left( \eta_t - \frac{\eta_t^2}{2} (5L^{(0)} + 4L^{(1)} \Vert \nabla f_t\Vert_\ast) \right) \Vert \nabla f_t\Vert_\ast^2 \\
    & = f_t - \frac{\Vert \nabla f_t\Vert_\ast^2}{2(5L^{(0)} + 4L^{(1)} \Vert \nabla f_t\Vert_\ast)}.
\end{split}
\end{equation}
If $\varepsilon \leq \Vert \nabla f_t\Vert_\ast \leq L^{(0)} / L^{(1)}$, we get
\begin{equation}
    f_{t+1} \leq f_t - \frac{\varepsilon^2}{18L^{(0)}}
\end{equation}
If $\Vert \nabla f_t \Vert_\ast \geq L^{(0)} / L^{(1)}$, we get
\begin{equation}
    \begin{split}
        f_{t+1} & \leq f_t - \frac{\Vert \nabla f_t\Vert_\ast^2}{2(5L^{(0)} + 4L^{(1)} \Vert \nabla f_t\Vert_\ast)} \\
        & = f_t -  \frac{\Vert \nabla f_t\Vert_\ast}{ 10L^{(0)} / \Vert \nabla f_t \Vert_\ast + 8L^{(1)}} \\
        & \leq f_t - \frac{\Vert \nabla f_t\Vert_\ast}{ 18L^{(1)}} \\
        & \leq f_t - \frac{L^{(0)}}{18 (L^{(1)})^2}
    \end{split}
\end{equation}
Hence,
\begin{equation}
    f_{t+1} \leq f_t - \min\left\{ \frac{L^{(0)}}{18 (L^{(1)})^2}, \frac{\varepsilon^2}{18 L^{(0)}}  \right\}.
\end{equation}
Now assume that we have $T$ iterations with $\Vert \nabla f_t\Vert_\ast \geq \varepsilon$.
Then
\begin{equation}
    f_0 - f^\star \geq f_0 - f_T = \sum_{t=0}^{T-1} (f_t - f_{t+1}) \geq T \min\left\{ \frac{L^{(0)}}{18 (L^{(1)})^2}, \frac{\varepsilon^2}{18 L^{(0)}}  \right\}.
\end{equation}
Rearranging yields
\begin{equation}
    T \leq 18 \frac{f_0 - f^\star}{\min\left\{ \frac{L^{(0)}}{(L^{(1)})^2}, \frac{\varepsilon^2}{L^{(0)}}\right\}} = 18 (f_0 - f^\star) \max \left( \frac{L^{(0)}}{\varepsilon^2},  \frac{(L^{(1)})^2}{L^{(0)}} \right).
\end{equation}

\end{proof}

\end{document}